\documentclass[final,11pt,noinfoline]{imsart}

\arxiv{}

\usepackage{booktabs}
\usepackage[english]{babel}
\usepackage[latin1]{inputenc}
\usepackage{textcomp}
\usepackage{dsfont}
\usepackage{color}
\usepackage{xcolor}
\usepackage{graphics}
\usepackage{graphicx}
\usepackage{epstopdf}
\usepackage{amsmath,amsthm,amssymb,amsfonts}
\usepackage{mathrsfs}
\usepackage[garamont]{mathdesign}
\usepackage{enumitem}
\usepackage[norelsize,ruled,vlined,commentsnumbered]{algorithm2e}
\usepackage{multirow}
\usepackage[final,activate,verbose=true,auto=true]{microtype}
\usepackage{a4wide}
\usepackage{bbm}
\usepackage{mathtools}
\usepackage{hyperref}
\usepackage{xcolor}
\usepackage{relsize}
\hypersetup{colorlinks=true,linkcolor=blue,citecolor=blue,linktoc=page}
\usepackage{hhline}
\usepackage{makecell}
\usepackage{diagbox}
\newcommand{\eq}{\begin{equation}}
\newcommand{\qe}{\end{equation}}

\newcommand{\N}{\mathbb{N}}                
\newcommand{\R}{\mathbb{R}}                     

\def\S{\mathbb{S}}
\def\spa{\operatorname{span}}

\def\O{\mathcal{O}}
\def\J{\mathcal{J}}
\def\R{\mathbb{R}}
\def\N{\mathbb{N}}
\def\S{\mathbb{S}}
\def\E{\mathbb{E}}
\def\P{\mathbb{P}}
\def\W{\mathcal{W}}
\def\V{\mathcal{V}}
\def\Id{\operatorname{Id}}
\def\I{\operatorname{Id}}
\def\H{\text{H}}
\def\Err{\operatorname{Err}}

\DeclarePairedDelimiter\ceil{\lceil}{\rceil}
\DeclarePairedDelimiter\floor{\lfloor}{\rfloor}

\newtheorem{theorem}{Theorem}
\newtheorem{lemma}[theorem]{Lemma}
\newtheorem{cor}[theorem]{Corollary}
\newtheorem{proposition}[theorem]{Proposition}

\newtheorem{remark}{Remark}

\date{\today}

\begin{document}
\sloppy

\begin{frontmatter}

\title{Relative concentration bounds for the spectrum of kernel matrices }
\runtitle{Relative concentration bounds for the spectrum of kernel matrices}

\begin{aug}
\author{\fnms{Ernesto} \snm{Araya Valdivia}\ead[label=e1]{ernesto.araya-valdivia@math.u-psud.fr}}
\affiliation{Université Paris-Sud}
\address{Laboratoire de Mathématiques d'Orsay (LMO)\\ Université Paris-Saclay \\ 91405 Orsay Cedex \\France}
\runauthor{Araya Valdivia}
\end{aug}

\begin{abstract}
In this paper we study the concentration properties for the eigenvalues of kernel matrices, which are central objects in a wide range of kernel methods and, more recently, in network analysis. We present a set of concentration inequalities tailored for each individual eigenvalue of the kernel matrix with respect to its known asymptotic limit. The inequalities presented here are of relative type, meaning that they scale with the eigenvalue in consideration, which results in convergence rates that vary across the spectrum. The rates we obtain here are faster than the typical $\O(\frac{1}{\sqrt n})$ and are often exponential, depending on regularity assumptions of Sobolev type. One key feature of our results is that they apply to non positive kernels, which is fundamental in the context of network analysis. We show how our results are well suited for the study of dot product kernels, which are related to random geometric graphs on the sphere, via the graphon formalism. We illustrate our results by applying them to a variety of dot product kernels on the sphere and to the one dimensional Gaussian kernel. 
 \end{abstract}



\end{frontmatter}

\maketitle 
\section{Introduction}

 Kernel methods have become nowadays an important tool in machine learning, with a wide range of applications including principal component analysis (PCA), clustering, non-parametric estimation and, more recently, statistical analysis of networks. Many of these methods rely on the spectral decomposition of a data dependent random matrix, which is commonly known as the \emph{kernel matrix}, constructed by evaluating a symmetric kernel on a set of sample points. In PCA for instance \cite[Sec.5.1]{Smola}, the original data is reduced to a low dimensional invariant subspace of a particular kernel matrix, based on the selection of a group of leading eigenvalues (those with larger value). An accurate estimation of the size of the eigenvalues of the kernel matrix is fundamental for obtaining theoretical guarantees for the error in this context.

The study of the spectra of kernel matrices, and the related sample covariance matrices, has a long history. In a seminal paper, Marchenko and Pastur \cite{Mar_Past} described the limit spectral distribution of random matrices of the form $\frac{1}{n}XX^T$, where $X$ is a $d\times n$ matrix with centered independent entries with finite variance, in the regime when the ratio $d/n$ converge to a fixed strictly positive number. Given that $d$ goes to infinity, this line of work is often referred to as the high-dimensional setting (typically the ambient space is Euclidean of dimension $d$ and $n$ is the sample size). In the high dimensional setting, El Karoui \cite{Karoui,Karoui2} studied the asymptotic and finite sample spectral properties of matrices of the form $f(XX^T)$, where $f$ is applied entrywise, and proved that under some regularity conditions on the kernel, the spectrum is essentially a linear deformation of the Marchenko-Pastur law. This results has been generalized, first in \cite{Chen} and later in \cite{Vanvu}, where many regularity hypothesis have been removed. 

In this paper we do not deal with the high dimensional setting. Instead we place ourselves in the so-called low-dimensional context (or fixed dimensional given that the data dimension is fixed, while the sample size grows) which has also attract interest lately. In \cite{Kolt}, the authors prove that under mild regularity (integrability) conditions, the spectrum of the kernel matrix, after proper normalization, converges to the spectrum of an infinite dimensional object: the $L^2$ integral operator associated with this kernel. The convergence is stated in terms of what they call $\delta_2$ metric, which is a $\ell_2$-type metric when the spectrum of the kernel matrix and the spectrum of the operator are regarded as elements of a sequence space (since we only deal with compact operators, this is possible). They also obtain a CLT describing the law of the fluctuation of the eigenvalues. Finite sample results for the same $\delta_2$ metric have later been obtained in \cite{Yohann}, where the authors study the problem of graphon estimation. Similar results have been obtained in the positive semidefinite case in \cite{Ros}. 

In the language of matrix perturbation, the $\delta_2$ metric is an example of an \emph{absolute} measure for the deviation of the eigenvalues, because it provides a uniform control over the spectrum. In relative bounds, the difference between eigenvalues of two matrices is weighted by a function of the eigenvalues themselves. For example, if $A$ and $B$ are two $n\times n$ symmetric matrices, and their eigenvalues are indexed from $1$ to $n$ decreasingly, then examples of relative measures for the deviation of the $i$-th eigenvalue  are $\frac{1}{\lambda_i(A)}|\lambda_i(A)-\lambda_i(B)|$ and $\frac{1}{\sqrt{\lambda_i(A)\lambda_i(B)}}|\lambda_i(A)-\lambda_i(B)|$. Those measures are often called of Weyl-type, because they involve single eigenvalues. There also exists relative measures for a group of eigenvalues (including the full spectrum) which are often obtained by adding the single eigenvalue bounds, for a set of indices. Those bound are often referred to as of Hoffmann-Weiland type, after the classic matrix eigenvalue inequality with that name. Relative inequalities are known for achieving better accuracy in both, the deterministic \cite{Ipsen} and the probabilistic setting \cite{Braun}. 

Our main results are relative concentration inequalities of the Weyl-type for kernels which are not necessarily positive, but which satisfy some regularity assumptions related to the pointwise convergence of the kernel's spectral expansion. In symbols, given a kernel $W$ we obtain inequalities of the form $|\lambda_i(T_W)-\lambda_i(T_n)|= O(|\lambda_i| n^{-q})$ where $q$ is a positive number (usually $0<q<1$), $T_W$ is an integral operator associated to $W$, $T_n$ is a normalization of the kernel matrix and the eigenvalues are in the decreasing order for their absolute value. This type of relative bound have also been called \emph{scaling bounds} \cite{Braun}, because the limit eigenvalue (which is an eigenvalue of the limiting integral operator) plays the role of a scaling term in the final error bound, and have an important effect on the estimated convergence rate. This scaling term will allow us to obtain rates that are better than parametric and often exponential (or almost exponential) in the sample size, similar to those obtained in \cite{Belkin} under different methods and with more restrictive hypothesis. Indeed, our results show that relative bounds are not only more accurate, but they also represent an alternative argument to the better than parametric rates presented in \cite{Belkin}. Formally, our rates have three factors: one scaling term, one variance term and one concentration term. In the case of exponential rates is the scaling term the one that prevails. We show how the effect of the scaling term on the eigenvalue convergence rates became more explicit when considering the regularity hypothesis on the kernel, which are related to the eigenvalue decay and the growth of the eigenfunctions in the $\|\cdot\|_{\infty}$ norm. We assume three type of regularity hypothesis which are common in the kernel literature and apply for a wide set of kernels. 

Many of the relative bounds or scaling bound in the literature are better adapted to the case when an index $i$ is fixed and describe the behavior of $|\lambda_i(T_W)-\lambda_i(T_n)|$ in terms of $n$ (as our first main result does). One of our contributions is to obtain inequalities for a varying $i$, which are specially useful when considering the smaller eigenvalues (in absolute value) or in the case of a fixed finite sample size. Our concentration bounds change depending on the index $i$ of the eigenvalue and they will exhibit mixed tail regimes. Our approach can be divided in three steps, which we call \emph{approximation}, \emph{perturbation} and \emph{concentration} steps, depending on the main techniques used in each one. In the approximation step, we consider a finite rank approximation of the kernel $W$ and the kernel matrix $T_n$ is decomposed accordingly in two parts: the truncated and the residual. In the perturbation step we use (deterministic) relative perturbation to bound $|\lambda_i(T_n)-\lambda_i|/|\lambda_i|$. The obtained bound depend on two random terms (one is related to the residual of the approximation step) that we call \emph{noise} terms. In the concentration step, we bound the noise terms using tail bound inequalities for $U$-statistics, such as those in \cite{GinLatZin}.

We apply our results to the study of a widely used family of radially symmetric kernels, which are known as \emph{dot product} kernels. We consider the Euclidean sphere $\S^{d-1}$ as the ambient space, equipped with the geodesic distance. In this case the following remarkable property holds: the eigenfunctions do not depend on the kernel itself, but only on the dimension of the sphere. In other words, the $L^2$ basis is fixed for every kernel in this family and coincides with the basis of spherical harmonics. Their growth rates are known, which fix one of the regularity parameters and the regularity hypothesis will depend on the kernel eigenvalues only. This simplify the main results and allow us to express them in terms of the dimension of the sphere and the eigenvalue decay rate. In this context, a bounded kernel can be identified with a graphon, which represent the limit of the sequences of dense graphs, following the theory of dense graph limits pioneered by Lovasz and collaborators \cite{Lovsze,Lova1,Lova3}. This makes relative concentration for single eigenvalues relevant in network analysis, specially given the recent development of methods and algorithms working with the spectrum in the contexts of testing \cite{Bub}, graphon estimation \cite{Yohann} and latent distance recovery in geometric graphs \cite{Ara}. This type of kernel has also found applications in the context of deep learning \cite{Cao}.

\subsection*{Overview}
The present paper is structured as follows: in Section \ref{Prelim} we introduce the basic material for integral operators and detail the related literature. In Section \ref{results} we state the main results and in Section \ref{sec:ideas_proof} we explain our three step approach and state the main propositions used in the proof of the main results. In Section \ref{Rate} we compare the obtained rates of convergence with results in the literature. In Section \ref{Hypo} we explain the connection between our regularity hypothesis for the kernel (decay-growth assumptions on the eigensystem) and the classical definition of Sobolev-type regularity. We explain in Section \ref{s:sect2} the connection between kernels and random graphs via the graphon model. In Section \ref{RG}, we adapt our results to the particular case of dot product kernels, and give some examples in the case of geometric random graphs, obtaining explicit expressions for the deviation of the spectrum. All the proofs are in the Appendix.
\subsection*{Notation}
We use the symbol ``$\lesssim$'' to denote inequality up to constants, that is: for $f,g$ real functions, $f(x)\lesssim g(x)$ iff $f(x)\leq Mg(x)$ for some $M>0$. Similarly, we use $f(x)\lesssim_\alpha g(x)$, which means that $f(x)\leq M(\alpha)g(x)$, to stress the fact that $M$ will depend on $\alpha$. We use the classic asymptotic notation in a similar way: $f(x)=\O(g(x))$ (resp. $f(x)=\O_\alpha(g(x))$ ) means that $f(x)\lesssim g(x)$(resp. $f(x)\lesssim_\alpha g(x)$), for $x$ larger than some $x_0$. We use $\|\cdot\|_{op}$ to denote the operator norm of matrices and/or operators. Given that we deal only with selfadjoint compact operators, the operator norm is defined as the largest singular value.

\section{Preliminaries}\label{Prelim}
Given a probability space $(\Omega,\mu)$, a kernel is defined as a symmetric measurable function $W:\Omega\times\Omega~\rightarrow~\mathbb{R}$. We will assume, here and thereafter, that $W$ is a square integrable with respect to the product measure $\mu\times\mu$. 
Each kernel $W$ has an associated integral operator $T_W~:~L^2~(\Omega,\mu)\rightarrow L^2(\Omega,\mu)$ defined by the relation \[T_W f(x)=\int_{\Omega}W(x,y)f(y)d\mu(y)\] By a classical result of functional analysis, $T_W$ is a compact operator(\cite[p.216]{Hirsch}) and by the spectral theorem for compact operators, the spectrum of $T_W$ is a numerable set, with $0$ as its only accumulation point (in consequence, every eigenvalue of $T_W$, excluding $0$, has finite multiplicity). We index the sequence of eigenvalues $\{\lambda_k\}_{k\in\mathbb{N}}$ in decreasing order with respect to its absolute value, that is $|\lambda_1|\geq |\lambda_2|\geq \cdots\geq 0$. Thus for $i\in\mathbb{N}$, $\lambda_i(T_W)$ will denote the $i$-largest eigenvalue in absolute value. On the other hand, we will use $\lambda(T_W)\in \mathbb{R}^{\mathbb{N}}$ to denote the infinite ordered spectrum (with the decreasing ordering defined above). Thus, the sequence $\lambda(T_W)$ can be seen as an element of $c_0$, the space of sequences that converge to $0$. In many cases it will be useful to consider the spectrum as a sequence, rather than an abstract set.

Another consequence of the spectral theorem is that the eigenvectors sequence $\{\phi_k\}_{0\leq k\leq\infty}$ form a Hilbertian basis of $L^2(\Omega,\mu)$. Furthermore, we can decompose the kernel in the $L^2(\Omega^2,~\mu~\times~\mu)$ sense as follows \begin{equation}\label{expan}
W=\sum_{k\in\mathbb{N}}\lambda_k\phi_k\otimes\phi_k
\end{equation}
where $\phi_k\otimes \phi_k(x,y)=\phi_k(x)\phi_k(y)$. We will say that a kernel $W$ has rank $R$ if the sequence $\{\lambda_k\}_{0\leq k\leq \infty}$ has exactly $R$ non-zero elements. If there are infinite non-zero eigenvalues, we say that the kernel has infinite rank. Stronger notions of converge hold for \eqref{expan}, under additional assumptions on $W$. For instance, in the classical Mercer theorem we have that the convergence in \eqref{expan} is uniform \cite{Mercer}, provided that $W$ is continuous and positive. 

Given $n$ i.i.d $\Omega$-valued random variables $\{X_i\}_{1\leq i\leq n}$ with common law $\mu$, we define the $n\times n$ matrix (with a slight abuse of notation) \[W_{ij}:=W(X_i,X_j)\] which we call the \emph{kernel matrix}(sometimes is called \emph{empirical matrix}). We are interested a normalized version of the kernel matrix, defined in terms of entries as $(T_n)_{ij}:=\frac{1}{n}W_{ij}$, which can be regarded as the empirical version of $T_W$. Notice that the matrix $T_n$ is symmetric, hence it has real eigenvalues. We use the notation $\lambda(T_n)$ and $\lambda_i(T_n)$ for the spectrum of $T_n$ in an analogous way as in the infinite dimensional case ($\lambda(T_n)$ is a sequence and $\lambda_i(T_n)$ a real number), completing the sequence $\lambda(T_n)$ with zeros to obtain a bonafide element of $c_0$. 
Our main objective is to quantify how similar the spectrum of $T_n$ and $T_W$ are, but for a single position $i$ and not for the whole sequence. For instance, how close is $\lambda_i(T_n)$ to $\lambda_i(T_W)$, which represent a pseudometric in the space $c_0$.

 Observe that the spectrum of $T_n$ depend on the random sample $\{X_i\}_{1\leq i\leq n}$, hence is random, while the spectrum of the operator $T_W$ is deterministic and fixed(do not depend on $n$). 
 
 \subsection{Related work}
 In \cite{Kolt} the authors study the asymptotic properties of $\lambda(T_n)$ using the $\delta_2$ metric, defined as follows. Given two sequences $x$ and $y$ on $\ell_2$, 
\[\delta_2(x,y):= \inf_{\pi\in P}\sqrt{ \sum^\infty_{i=1}(x_i-y_{\pi(i)})^2}\]
where $P$ is the set of permutations with finite support.
The almost sure convergence  $\delta_2(\lambda(T_n),\lambda(T_W))\rightarrow 0$ as $n\rightarrow \infty$ is proven. They also derived a central limit theorem, giving the law of fluctuations of $\lambda(T_n)$ around $\lambda(T_W)$. In \cite{Yohann} the authors obtained a concentration bound for $\delta_2(\lambda(T_n),\lambda(T_W))$ with respect to its mean and apply this to the non-parametric estimation of graphons. The authors prove that a non-parametric rate is achieved, which depend on the regularity (of the Sobolev type) of the kernel in consideration. In \cite{Ros}, the positive definite case is considered and a parametric rate is obtained. Given that $|\lambda_i(T_W)-\lambda_i(T_n)|\leq \delta_2\big(\lambda(T_W),\lambda(T_n)\big)$, those results already give an upper bound for the single eigenvalue convergence rate and serve as benchmarks. However, this bound gives a uniform control on the eigenvalues deviation, which in many cases is too rough and, in consequence, we expect an overestimation of $|\lambda_i(T_W)-\lambda_i(T_n)|$, specially for the smaller eigenvalues as noted in \cite{Braun}. 

In \cite{Braun} the authors obtain relative concentration bounds for single eigenvalues. They call them \emph{scaling bounds} because the error term scales with the eigenvalue in consideration, which frequently guarantees a more accurate estimation. Their results assume the positive definiteness of the kernel, which as we mentioned, might be restrictive for some application, notably in the context of networks. The rates we obtain are not only an improvement over those in \cite{Braun}, but also work in the more general framework of indefinite kernels. In particular, we prove that certain term of bias introduced by the truncation approach (which consists in approximate an infinite rank kernel for a finite rank one) do not have an impact in the rate for a great portion of the spectrum and we have at least parametric rate (both elements are not present in \cite{Braun}). 

In \cite{Belkin} the authors take an approximation theoretic point of view, proving that for very regular (infinitely differentiable) radially symmetric kernels, the spectrum of the integral operator $T_W$ decays (almost) exponentially and, in addition, the eigenvalues of the kernel matrix are almost exponentially close to those of $T_W$. Those results are in line to the ones we obtain in the case of exponential decay of the eigenvalues. Results in \cite{Belkin} are however obtained using Hilbert space methods, which in principle are only valid for positive kernels (which can be identified with an RKHS) and it is not clear how to extend them to the indefinite case.

There are many results dealing with the related subject of concentration for sample covariance matrices (or operators) such as \cite{KoltLou,Rudi,Mortiz}. In these works, relative concentration inequalities of the type $|\lambda_i(S)-\lambda(\hat{S})|\leq \epsilon \lambda_i(S)$ are proven, where $S$ (resp. $\hat{S}$) is the population (resp. empirical) covariance matrix. They obtain results of the form $\epsilon=\O(r(S)n^{-1/2})$, where $r(S)$ is the effective rank and is proportional to $\sqrt{\operatorname{rank}(S)\log{d}}$. The main difference with our results is that in general they consider random vectors with a fixed sub-Gaussian norm, which here is not the case. The hypothesis of fourth moments, considered for example in \cite{Rudi}, is not well adapted to the context of dot product kernels(see Sect. \ref{RG}). In addition, a key element of some of the results in this line is the use of the RKHS technology which do not have direct extension for the indefinite case.

\section{Main results}\label{results}

For our main result we introduce a summability hypothesis for the spectral expansion of the kernel given in \eqref{expan}. This guarantees the pointwise convergence of the expansion necessary when evaluating the eigenfunctions $\{\phi_k\}^\infty_{k=1}$ in the sample points $\{X_i\}^n_{i=1}$. We also introduce three, more specialized, hypothesis that are common in the kernel literature, which are the polynomial decay of the eigenvalues, on one side, and the exponential decay, on the other. These hypotheses are complemented with growth hypothesis on the eigenvectors, which together will imply the summability of the kernel spectral expansion. This hypotheses are satisfied by many widely used kernels.

More formally, we will assume that the kernel $W$ satisfies $\|\sum^\infty_{i=1}|\lambda_i|\phi^2_k\|_\infty\leq \infty$ which we call hypothesis $\text{H}$. In Lemma \ref{lem:convergenceas} we prove that $\text{H}$ implies that the spectral expansion of the kernel converges $\mu\times\mu$ almost surely. This is the main hypothesis for Theorem \ref{thm:theo1} below. Observe that under $\H$, the operator $T_W$ is trace-class, that is $\sum^\infty_{i=1}|\lambda_i|<\infty$.  Notice that this hypothesis is stronger than the bounded kernel assumption, but weaker than assuming a Mercer kernel, which is a common assumption in the kernel literature. Since our results are in high probability, the uniform converge, given for instance by the Mercer theorem, is not necessary. In Theorem \ref{thm:theo2} we will  assume that for $i\in\N$ that $|\lambda_i|=\O(f(i))$ where $f$ is either polynomial $f(i)=i^{-\delta}$ (hypothesis $\text{H}_1$) or exponential $f(i)=e^{-\delta i}$ (hypothesis $\text{H}_2$ and $\text{H}_3$). These two rates regimes are common assumptions in theoretical analysis of kernel methods, see \cite{Braun,Xu,Rudi}. 

To quantify the possible growth of the eigenvectors $L^\infty(\Omega,\infty)$ norm, we will assume $\|\phi_i\|_{\infty}=\O(i^{s})$ under $\text{H}_1$ and $\text{H}_2$. In the case of $\text{H}_3$, we will assume that $\|\phi_i\|_{\infty}=\O(e^{is})$. We do not consider the case polynomial decrease in the eigenvalues and exponential increase of the eigenfunctions, since in this case the kernel expansion series diverge and it is incompatible with $\H$. Indeed, the hypothesis $\H$ impose restrictions on the possible values of $\delta$ and $s$, which are summarized in the Table \ref{tablehyp} below. We will always assume that $\delta,s\in\N$, which is not necessary in our proofs, but it allow us to have consistency with the classical definition of regularity as explained in Section \ref{Hypo}.


\begin{table}[h!]
\centering
\begin{tabular}{c|l*{4}{c}}
&  \multicolumn{ 3}{c}{Assumption} \\ \Xhline{2\arrayrulewidth}
$\text{H}$ & \multicolumn{ 3}{c}{$\|\sum_{i}|\lambda_i|\phi_i^2\|_\infty<\infty$}\\  \hline\hline
& \multicolumn{1}{c|}{$|\lambda_i|$}&   \multicolumn{1}{c|}{$\|\phi_i\|_\infty$}&  \\ \Xhline{1.5\arrayrulewidth}
$\text{H}_1$&  \multicolumn{1}{c|}{$\O(i^{-\delta})$} &  \multicolumn{1}{c|}{$\O(i^s)$} & $\delta>2s+1$ \\ \hline
$\text{H}_2$ & \multicolumn{1}{c|}{$\O(e^{-\delta i})$} &  \multicolumn{1}{c|}{$\O(i^s)$} &$\delta>s$\\  \hline
$\text{H}_3$ &  \multicolumn{1}{c|}{$\O(e^{-\delta i})$}&   \multicolumn{1}{c|}{$\O(e^{si})$} & $\delta>2s$
\end{tabular}
\caption{Hypotheses for the eigenvalue decay and the growth of the eigenvectors}
\label{tablehyp}
\end{table}

The following theorem works under the $\text{H}$ hypothesis. We define \[\mathcal{V}_1(i):=\|\sum^i_{k=1}\phi^2_k\|_{\infty},\] which works as variance proxy, in the sense of concentration inequalities \cite{Massart}, in what follows. 
\begin{theorem}\label{thm:theo1}
Let $W:\Omega\times \Omega\rightarrow [0,1]$ be a kernel which eigensystem satisfies $\text{H}$. Fix $i\in \N$ and define \[R(i):=\min\big\{R\in\N:|\lambda_i|>\sum_{k>R}|\lambda_k|\vee \sqrt{R\sum_{k>R}\lambda^2_k}\big \}\] Then there exists $n_0\in \N$ such that for $n\geq n_0$ and for $\alpha\in(0,1)$ we have \[|\lambda_i(T_n)-\lambda_i|\lesssim |\lambda_i|\sqrt{\frac{\mathcal{V}_1(R(i))\log{R(i)/\alpha}}{n}}\] with probability larger than $1-\alpha$.
\end{theorem}

\begin{remark}
From the fact that $T_W$ is trace-class, it follows that $\sum_{k>R}|\lambda_k|$ and $R\sum_{k>R}\lambda^2_k$ converge to $0$ as $R$ tends to infinity (see the proof of Theorem \ref{thm:theo1} in Sec.\ref{app:proof_thm1}). Consequently the set in the definition of $R(i)$ is non-empty. In Table \ref{tableRi} we show bounds for the term $R(i)$, under the polynomial and exponential decay assumptions for the eigenvalues. 
\end{remark}

The $n_0$ that appears in Theorem \ref{thm:theo1} depends on $i$ in general. This dependence will be made explicit in the proof (Appendix \ref{app:proof_thm1}). One of the consequences of Theorem \ref{thm:theo1} is that $|\lambda_i(T_n)-\lambda_i|$ attains a parametric rate in terms of $n$, when $i$ is fixed and $n$ is large enough, while maintaining the scaling term $|\lambda_i|$. The latter will be fundamental to obtain faster rates, under assumptions $\H_1,\H_2,\H_3$. Stated this way, this result is close to the CLT proven in \cite{Kolt}, which says  that when the eigenvalues of $W$ are simple (multiplicity one) then the following convergence in law holds $\lambda_i(T_n)\to \lambda_i(T_W)G_\mu(\phi_i^2)n^{-1/2}$, where $G_\mu$ is the generalized Brownian bridge associated with $\mu$ (a centered Gaussian process indexed by $L^2$ functions whose covariance is same as defined by $\mu$). Note that Theorem \ref{thm:theo1} in this respect is close to this asymptotic result, except that the variance term $\mathcal{V}_1(R(i))$ involves not only $\phi_i^2$, but all the functions up to $R(i)$. This variance term $\V_1(R(i))$ is similar to those appearing in the (absolute) bounds in positive kernel literature \cite{Shawe} and \cite{Blanchard}. In the aforementioned results, the variance term is the radius of smaller ball, in a Hilbert space, that contains the evaluation of feature maps. Here is the radius of the smaller ball, in the $\|\cdot\|_{\infty}$ norm, that contains the evaluation of the eigenfunctions.

The following theorem works under the more specialized hypothesis $\text{H}_1,\text{H}_2$ and $\text{H}_3$. 


\begin{theorem}\label{thm:theo2}
 Let $W$ be a kernel satisfying one of the hypothesis $\text{H}_1,\text{H}_2$ or $\text{H}_3$. Then with probability larger than $1-\alpha$ we have a bound of the form \[|\lambda_i(T_n)-\lambda_i|\lesssim B(i,n)\log{1/\alpha}\] where $B(i,n)$ depends on the respective hypothesis and is given by the following table
 

 \begin{table}[h!]
    \centering
    \begin{tabular}{|l*{3}{|c|}}
    \hline
    Assumption & $B(i,n)$ & $i$  \\ \hline
     \multirow{ 3}{*}{$\text{H}_1(s\geq 1)$ }    & $i^{-\delta+\frac{\delta}{\delta-1}(s+\frac12)}  n^{-\frac12}$&  $1\leq i \leq n^{\frac{\delta-1}{\delta}\frac1{2s+1}}$\\ 
     & $i^{-\delta +1+\frac{\delta-1}{\delta}(s+\frac12)}n^{-\frac 12}$ &$n^{\frac{\delta-1}{\delta}\frac1{2s+1}}\leq i\leq n^{\frac{1}{2s}}$ \\ 
     &$i^{-\delta +s+1}n^{-\frac 12}$ &$ n^{\frac{1}{2s}} \leq i\leq n$\\ \hline
   $\text{H}_1(s=0)$& $i^{-\delta +\frac12}n^{-\frac 12}$ &$1\leq i\leq n$   \\ \hline
      \multirow{ 2}{*}{$\text{H}_2(s\geq1)$ }    & $e^{-\delta i+(s+\frac12)\log i}  n^{-\frac12}$ &  $1\leq i \leq n^{\frac1{2s}}$ \\ 
     & $e^{-\delta i+s\log i} n^{-\frac12}$ &  $ n^{\frac1{2s}}\leq i\leq n$ \\ \hline
      $\text{H}_2(s=0)$& $e^{-\delta i+\frac12\log i}n^{-\frac 12}$ &$1\leq i\leq n$   \\ \hline
      $\text{H}_3(s\geq 1)$    &$e^{(-\delta +s)i}  n^{-\frac12}$& $1\leq i \leq n$ \\ \hline 
    \end{tabular}
\end{table}

 \end{theorem}
 

Both theorems derive from the same set of results (this is explained in Section \ref{sec:ideas_proof} below), but they are adapted for different situations. In Theorem \ref{thm:theo1} we have a fixed index $i$ and we allow the sample size to grow, obtaining a rate in terms of $n$. In Theorem \ref{thm:theo2}, we assume a fixed sample size and allows $i$ to vary with $n$. Observe that the obtained rates change depending on the value of $i$ (relative to $n$). This is known phenomena in concentration inequalities, where often we have a mix between different tail regimes (frequently Gaussian and Exponential). In Theorem \ref{thm:theo1} this is less important as the focus is when $n$ is large, while $i$ is fixed, and the term $\O(1/\sqrt n)$ prevails. Observe that Theorem \ref{thm:theo2} gives rates that are better than parametric under $\H_1$ and exponential in cases $\H_2$ and $\H_3$. This in line with some recent results such as \cite[Thm.2]{Belkin}. A more detailed comparison with previous results in the literature is postponed to Section \ref{Rate}. In the next section we explain the ideas behind the proof of the main results. 
\begin{table}
\centering
\begin{tabular}{c|l*{2}{c}}
Assumption &$R(i)$\\ \hline
$\lambda_i=\O(i^{-\delta})$& $\O(i^{\frac\delta{\delta-1}})$\\ 
$\lambda_i=\O(e^{-i\delta})$& $\O(i)$
\end{tabular}
\caption{}
\label{tableRi}
\end{table}

\subsection{Three step approach: approximation, perturbation and concentration}\label{sec:ideas_proof}

The proofs of Theorems \ref{thm:theo1} and \ref{thm:theo2} are based on three steps: approximation, perturbation and concentration. In the approximation step we construct a finite rank version of $T_W$ which allows a direct comparison with $T_n$. This gives a framework where $T_n$ can be seen as a random perturbation of $T_W$. In the perturbation step, we use deterministic matrix perturbation to quantify the deviation between $\lambda(T_W)$ and $\lambda(T_n)$. At the end of this step we obtain a scaling bound that depend on random error terms (which are written as the operator norm of random error matrices). Finally, in the concentration step we use concentration inequalities for $U$-statistics to control the error terms. Throughout this section, we will assume that $\text{H}$ holds.

 \textbf{Approximation step:} The approximation step builds upon a truncation approach, used for example in \cite{Kolt,Yohann,Braun}, where we fix $R\in \N$ and decompose $W$ into two terms: the best rank $R$ approximation of $W$, in the $L^2$ sense, and the residual term. More specifically, 
we define
\begin{equation}\label{finapprox}W_R(x,y):=\sum^R_{k=1}\lambda_k\phi_k(x)\phi_k(y)
\end{equation}
We call $R$ the truncation parameter. Thus following decomposition holds \[(T_n)_{ij}=\frac1n W_R(X_i,X_j)+\frac1n(W-W_R)(X_i,X_j)\]
The previous equality holds in the almost sure sense, given $\text{H}$ and Lemma \ref{lem:convergenceas}. Here and thereafter we will work on the set of measure one where this equality holds true. We call the first term of the right hand side, the $R$-\emph{truncated kernel matrix}. Define the \emph{residual matrix} (the error from the approximation) as \[(E_R)_{ij}:=\frac1n (W-W_R)(X_i,X_j)=\sum_{k>R}\lambda_k\phi_k(X_i)\phi_k(X_j)\] where the second equality is justified by the assumption $\text{H}$, which implies the pointwise equality. The $R$-truncated kernel matrix can be written as a multiplicative perturbation of a diagonal matrix. More specifically, we have the following factorization \[\frac{1}{n}W_R(X_i,X_j)=\Phi_R\Lambda_R\Phi_R^T\] where $\Phi_R$ is the $n\times R$ matrix with columns $1/\sqrt n(\phi_k(X_1),\phi_k(X_2),\cdots,\phi_k(X_n))^T$, for $1\leq k\leq R$, and $\Lambda_R$ is a diagonal matrix with $\lambda_1,\lambda_2,\cdots,\lambda_R$ in the diagonal. Thus, the normalized kernel matrix can be written as an additive perturbation of the $R$-truncated matrix by the residual matrix 
\begin{equation}\label{eq:pert1}
T_n=\Phi_R\Lambda_R\Phi^T_R+E_R
\end{equation}
From the previous equality is already possible to obtain scaling bounds, as is done in \cite{Braun}. The idea is that the first term has a structure that makes it compatible with standard tools of matrix concentration (after using a deterministic multiplicative perturbation theorem). The residual term has less structure, but its operator norm will be small in comparison to the $R$-truncated matrix, provided that $R$ is well chosen. The fact that neither the $R$-truncated matrix nor the residual are positive definite and invertible, prevents the use of most of the relative Weyl type bounds, see \cite{Ipsen} for example. In \cite{Braun} and \cite{Yohann} the classic (absolute) Weyl inequality is used. Intuitively speaking, the problem with that approach is that in absolute perturbation we consider the effect of the residual term over all the eigenvalues of the truncated matrix is uniform. In light of the asymptotic results in \cite{Kolt}, given the $L^2$ eigenfunctions orthogonality, this should not be the case. To overcome this, we introduce another factorization instead, which allows to better exploit the multiplicative perturbation framework. The factorization is as follows 
\begin{equation}\label{eq:res1}T_n = (\Phi_R|\Phi^\perp_R)M(\Phi_R|\Phi^\perp_R)^T+A\end{equation}
where \[M= \begin{pmatrix}
   \Lambda_R & 0\\\
0 & M_{>R}
 \end{pmatrix}\] the columns of matrix $\Phi^\perp_R$ are an orthonormal basis of the orthogonal complement to the space spanned by the columns of $\Phi_R$. Assume for the moment that the columns of $\Phi_R$ are linearly independent. On that event, define the projection matrices $P_1:=\Phi_R(\Phi^T_R\Phi_R)^{-1}\Phi_R^T$ and $P_2:=\Phi_R^\perp{\Phi_R^\perp}^T$, the matrices $M_{>R}$ and $A$ are specified by  \begin{align*}
    M_{>R}&:= {\Phi^\perp_R}^TE_R\Phi^\perp_R\\
    A&:=P_1E_RP_2+P_2E_RP_1+P_1E_RP_1
\end{align*} 
Notice that the columns of $\Phi_R^\perp$ are orthonormal, which is not the case of the columns of $\Phi_R$. In words, we decompose the residual matrix according to its projection onto the space generated by the columns of $\Phi_R$ and its orthogonal complement. 

\noindent \textbf{Perturbation step}: 
Define $\tilde{M}:=(\Phi_R|\Phi^\perp_R)M(\Phi_R|\Phi^\perp_R)^T$.  We first use Weyl's perturbation theorem to obtain the following \begin{equation}\label{eq:pert_1}|\lambda_i(T_n)-\lambda_i(\tilde{M})|\leq \|A\|_{op}\end{equation}

 Here we use a relative multiplicative perturbation theorem known as Ostrowskii's inequality \cite[Thm.4.5.9, Cor. 4.5.11]{Horn} (see \cite[Cor.A.3]{Braun} for the non-square case and see Remark \ref{rem:order_con} for the applicability with a different indexing than the non-increasing order) to obtain for all $1\leq i\leq n$ 
\begin{equation}\label{eq:res2}
 |\lambda_i(\tilde{M})-\lambda_i(M)|\leq |\lambda_i(M)|\|(\Phi_R|\Phi^\perp_R)^T(\Phi_R|\Phi^\perp_R)-\I_n\|_{op}=|\lambda_i(M)|\|\Phi_R^T\Phi_R-\I_R\|_{op}\end{equation} where the last equality comes from the fact that $\Phi_R^\perp$ has orthonormal columns. 
Using \eqref{eq:res2} and \eqref{eq:pert_1} we obtain for all $1\leq i \leq n$
\begin{align}\label{eq:pert_fin1}
|\lambda_i(T_n)-\lambda_i(M)|\leq |\lambda_i(M)|\|\Phi^T_R\Phi_R-\I_R\|_{op}+\|A\|_{op}
\end{align}
 Because of the block structure of $M$, we have that $\lambda(M)=\lambda(\Lambda_R)\cup \lambda(M_{>R})$. The eigenvalues $\lambda(\Lambda_R)$ are deterministic, while $\lambda(M_{>R})$ is a random set. By the definition of $M_{>R}$, the following trivial bound holds\[|\lambda_i(M_{>R})|\leq \|{\Phi^\perp_R}^TE_R\Phi^\perp_R\|_{op},\text{  for } 1\leq i\leq n-R\]

\textbf{Concentration step:} 
We use concentration inequalities to control with high probability the terms in the right hand side of \eqref{eq:pert_fin1} and will also serve us to characterize the random set $\lambda(M_{>R})$. As we already mention, inequalities for quantities of the form $\|\Phi_R^T\Phi_R-\I_R\|_{op}$ are well established. For instance, given that $\Phi_R^T\Phi_R$ can be written as a sum of independent random matrices, we can use the non-commutative Bernstein inequality. For instance, using the matrix Bernstein inequality \cite[Thm 6.1]{Tropp} we obtain 

\begin{proposition}\label{prop:con_Gram}
With probability larger than $1-\alpha$ we have 
\[ \|\Phi_R\Phi_R^T-\I_R\|_{op}\lesssim \frac{\V_1(R)\log{R/\alpha}}{n}\vee\sqrt{ \frac{\V_1(R)\log{R/\alpha}}{n}}\]
\end{proposition}

For the terms $ \|A\|_{op}$ and $\|\Phi^\perp_RE_R{\Phi^\perp_R}^T\|_{op}$, the standard tools in matrix concentration inequalities, do not seem to apply as smoothly. For instance, when $E_R$ has infinite rank, it cannot be expressed as a finite sum of independent matrices. Some concentration inequalities, such a the matrix bounded differences inequality \cite[Cor. 6.1]{Mckey_1} or others obtained by the matrix Stein method \cite{Mckey_2} can be applied in this case, but they demand strong conditions such as almost sure control of a matrix variance proxy in the semi-definite order. In addition, in this case they deliver suboptimal results. We opt for using a rougher matrix norm inequality, for example using the Frobenius norm to control the operator norm, and we then use concentration inequalities for $U$-statistics, such as those in \cite[Thm.3.3]{GinLatZin} or \cite[Thm.3.4]{ReyBour}. This is can be seen as rough application of the \emph{comparison method} described in \cite{VanHandel}, which consists in find an easier-to-bound random process majorizing the operator norm. Finding a tight majorizing random process is, in general, a challenging task and no canonical way to do this is known. The fact that we use a rougher bound for the matrix norm will be compensated by optimizing the choice of $R$, which helps reducing the impact of this inaccuracy. 

We have the following proposition which gives a tail bound for the terms $ \|E_R\phi\|_{op}$, for $\phi\in\{\phi_1,\cdots,\phi_R\}$, and $ \|{\Phi^\perp_R}^TE_R{\Phi^\perp_R}\|_{op}$. 

\begin{proposition}\label{prop:tail_bounds}
We have with probability larger than $1-\alpha$ \begin{align}\label{eq:prop_1}
 \sqrt{\sum^R_{l=1}\|E\phi_l\|^2_{op}}& \lesssim_\alpha \sqrt{\frac{1}{n}b_{2,R}\V'_1(R)}=:\gamma_1(n,R)\\
   \label{eq:prop_2}  \|{\Phi^\perp_R}^TE_R\Phi^\perp_R\|_{op}&\lesssim_\alpha b_R+\sqrt{\frac{\V_2(R)b_R}{n}}\vee\frac{\V_2(R)}{n}=:\gamma_2(n,R) 
    \end{align}
    where \[
    b_R:=\sum_{k>R}|\lambda_k|,\quad
    b_{2,R}:=\sum_{k>R}\lambda^2_k,\quad
    \V'_1(R):= \sum^R_{k=1}\|\phi_k\|^2_{\infty},\quad
    \V_2(R):= \|\sum_{k>R}\lambda_k\phi_k\otimes\phi_k\|_{\infty}
\]
\end{proposition}
Define the event \[\mathcal{E}_{\tau}:=\{\|\Phi_R^T\Phi_R(\omega)-\I_R\|_{op}<\tau\}\]
Define $\tau_{n,R,\alpha}:=\sqrt{\frac{\V_1(R)\log{R/\alpha}}{n}}$. The following lemma, which proves that the event $\mathcal{E}_{\tau_{n,R,\alpha}}$ holds with high probability, is a direct consequence of Proposition \ref{prop:con_Gram}.  
\begin{lemma}\label{lem:lemmaEvent}
For $R<n$ we have \[\P\Big(\mathcal{E}_{\tau_{n,R,\alpha}}\Big)\geq 1- \alpha\] 
\end{lemma}

\begin{lemma}\label{lem:tau}
Let $\operatorname{Sp}(\Phi_R)$ be the linear span of $\phi_1,\cdots,\phi_R$. It holds \[\|A\|_{op}\lesssim \max_{\phi\in\operatorname{Sp}(\Phi_R),\|\phi\|=1}\|E\phi\|\]
Let $\mathcal{E}_\alpha$ be the event such that \eqref{eq:prop_1} holds. In the event $\mathcal{E}_\alpha\cap \mathcal{E}_{\tau_{n,R,\alpha}}$ we have \[\|A\|_{op}\lesssim_\alpha \frac{1}{1-\tau_{n,R,\alpha}}\gamma_1(n,R)\]
\end{lemma}


For $\alpha\in(0,1)$ and $R<n$ we have, combining Lemma \ref{lem:tau}, Lemma \ref{lem:lemmaEvent} and Prop. \ref{prop:tail_bounds}\begin{equation}\label{eq:boundA}\|A\|_{op}\lesssim_\alpha \frac{1}{1-\tau_{n,R,\alpha}}\gamma_1(n,R)\end{equation} with probability larger than $1-2\alpha$.
The following two propositions help us to control $|\lambda_i(T_n)-\lambda_i|$ for a fixed $R\in\N$.

\begin{proposition}\label{prop:Rsmaller}
 Assume that $R\in\N$ is such that $\tau_{n,R,\alpha}<1$. 
Then with probability larger than $1-\alpha$ we have, for $i<R$
\begin{equation}
\label{eq:prop_prin}
|\lambda_i(T_n)-\lambda_i|\lesssim_{\alpha,\tau}\big(|\lambda_i|\vee\gamma_2(n,R)\big)\sqrt{\frac{\V_1(R)\log{R}}{n}}+\gamma_1(n,R)
\end{equation}

\end{proposition}

\begin{proposition}\label{prop:Rlarger}
 Fix $R\in\N$. We have, with probability larger than $1-\alpha$ for $i>R$ 
\[|\lambda_i(T_n)-\lambda_i|\lesssim_\alpha \gamma_2(n,R)\]
\end{proposition}

The proof of Theorem \ref{thm:theo1} uses Proposition \ref{prop:Rsmaller}. Indeed, it is easy to see that $\gamma_2(n,R)\to0$ when $R\to\infty$ and $\gamma_2(n,R)\to b_R$ when $n\to\infty$. On the other hand, we have $\gamma_1(n,R)\to 0$ when $n\to\infty$. For a fixed $i$ and for $n$ large enough, we can always choose $R$ to satisfy $\lambda_i>\gamma_2(n,R)$ and $\lambda_i>\gamma_1(n,R)$, then the Proposition \ref{prop:Rsmaller} will imply the bound in Theorem~\ref{thm:theo1}. 
For Theorem~\ref{thm:theo2}, on the other hand, we use either Prop. \ref{prop:Rsmaller} or Prop. \ref{prop:Rlarger} depending on the relative position of $i$ with respect to $n$. The fact that we have explicit assumptions on the eigenvalues and eigenfunctions allow us to make the relation between $\lambda_i$, $\gamma_1(n,R)$ and $\gamma_2(n,R)$ more precise (see Lemmas \ref{lem:rel1}-\ref{lem:tau_H1}, in the Appendix).

\begin{remark}[Ostrowski's inequality for the decreasing absolute value ordering]\label{rem:order_con}
As we mentioned above, the Ostrowski's inequality is formulated in the case of non-decreasing (or equivalently non-increasing) ordering. Nonetheless, it is still valid for the decreasing ordering in the absolute value. Indeed, if $\{\lambda_{\sigma(i)}\}_{1\leq i\leq n}$ is an ordering of the eigenvalues, that is~$\sigma:[n]\rightarrow [n]$ is bijective, we can reorder them in the non increasing order by applying a transformation $\sigma^\uparrow$ to each $\sigma(i)$, then apply the Owstroski's inequality and finally apply~${\sigma^\uparrow}^{-1}$. The key here is that this reordering process is applied to a finite matrix, because some orderings are not compatible with the operator $T_W$ full spectrum (the increasing ordering cannot be applied to the spectrum of an indefinite operator, given that $0$ is an accumulation point). 
\end{remark}

\section{Asymptotic rate analysis and further related work} \label{Rate}

The rates obtained in Theorem \ref{thm:theo2} are expressed in terms of $i$ and $n$, which is natural for a fixed $i$, or purely in terms of $n$, by replacing $i=n^{\log{i}/\log{n}}$. Under $\text{H}_1$, we obtained a parametric rate in terms of $n$. Indeed, we can decompose this rate in  $i^{-\delta+(s+1/2)g(i)}$(scaling and variance term), where $g(i)\leq 2$, and a concentration term $n^{-1/2}$. This goes in line with the CLT in \cite{Kolt}, where the same scaling and concentration terms appear. In that sense, the scaling and concentration terms seem to be optimal, while there is still room for some improvement in the variance term. If we allow $i$ to vary with $n$, we obtained rates that are fully expressed in terms of $n$, in which case they are always faster than $\O(n^{-1/2})$ for all three hypotheses $\text{H}_1$, $\text{H}_2$ and $\text{H}_3$.

In \cite{Lou} and \cite{Rudi} formally similar relative bounds are proven, which are in line with sample covariance concentration results in \cite{KoltLou}. In those articles the authors obtain bounds, for the difference of the empirical eigenvalue $\hat{\lambda}_i$ and the population eigenvalue $\lambda_i$ of covariance matrices, which are of the form $|\hat\lambda_i-\lambda_i|\lesssim \lambda_i\sqrt{\frac{r(S)}{n}}$, where $r(S)=Tr(S)/\|S\|_{op}$ and $S$ is the population covariance matrix. In the case of \cite{Rudi}, a different variance term is introduced, which depend on a regularization step based on shifting up the eigenvalues of $S$. Those results are similar in form to Theorem \ref{thm:theo1}, but the variance term differ. The term is $r(S)$ is introduced as a measure of effective rank of the covariance matrix (or operator). Observe that $r(S)$ is the same for all indices and do not change with $i$, which is a difference with our approach. This is partly explained by the different context their work is devoted to. Indeed, one of the main assumptions is that the random vectors, of which $S$ is the population covariance, have a fixed subgaussian norm which do not change in terms of the ambient dimension. This is not the case in our context, as explained in Section \ref{sec:ideas_proof} above, given that we do not have a fixed Euclidean space where random vectors are sampled. Otherwise stated, the random vectors we need control are not in a fixed $\R^d$ with a fixed sub-Gaussian norm, but defined in $\R^R$ and the truncation parameter $R$ has to be determined and the sub-Gaussian norm of the columns of $\Phi_R$ will depend on $R$. In \cite{Rudi}, heavy tailed random vectors are considered, but the fourth moments assumption is not well adapted to our context (many of the kernels described in Section \ref{RG} will not satisfy that condition, because of the $L^4$-norm growth of the spherical harmonics \cite{Han}). On the other hand, the works \cite{Lou,Rudi,Mortiz} are defined in the case of a positive matrix. Extensions to the case of non necessarily positive operators seem feasible, but they are not directly developed. 

The asymptotic rates obtained in \cite{Braun}, for positive operators, are slower than those given in Theorem \ref{thm:theo2} under the three hypothesis $\text{H}_1$, $\text{H}_2$ and $\text{H}_3$. They do not assume explicit growth rates, but formulate their result under the assumption of bounded eigenfunctions (which falls in our framework with $s=0$) and bounded kernel function(which is implied by $\text{H}$). For example, in the case of polynomial decay of the eigenvalues and bounded kernel, they obtain an error rate of $\O(n^{\frac{1-\delta}{2\delta}}\sqrt{\log{n}})$, which is slower than $\O(n^{-1/2})$ in terms of $n$. We observe that for $i$ fixed we obtain a better rate for the error in terms of $n$. Indeed, the rate in Theorem \ref{thm:theo2} under $\text{H}_1$ is $\O(n^{-\frac{\log{i}}{\log n}(\delta-1)-1/2)})$, which is faster than parametric provided that $\delta>1$ (which we have to assume in order to satisfy $\text{H}$ and which it is also assumed in \cite{Braun}). Similar comparison can be stablished in the case of exponentially decay eigenvalues. More importantly, we avoid the cumbersome bias terms present for example in \cite[Thm. 3]{Braun}.

The rate obtained in \cite[Thm.2 ]{Belkin} for $|\lambda_i(T_n)-\lambda_i(T_W)|$ is $\O(e^{-ci^{1/d}})$, where $c$ is a positive constant and $W$ is a positive, radially symmetric, infinitely differentiable kernel defined on $\R^d$. Their result is obtained by using approximation theoretic methods and it is a measure independent bound, from which the aforementioned rate can be easily deduced. This represents an intermediate regime between $\text{H}_1$ and $\text{H}_2$. Their rate do not seem to depend explicitly on the eigenfunctions growth, however the fact that the kernel is highly regular and radially symmetric would have an effect (this shares similarities with the case of dot product kernels presented in Section \ref{RG} below). At least formally, our results are aligned with those in \cite{Belkin}, in the sense that when an exponential rate of the eigenvalues of $T_W$ is observed, the deviation $|\lambda_i(T_n)-\lambda_i(T_W)|$ will have an exponential rate (the scaling term prevails over the concentration term). It is worth mentioning that the approximation theoretic methods used in \cite{Belkin} rely heavily on the RKHS technology and do not extend automatically to the non positive case. Extension of this approach to the indefinite case, relaxing the symmetric and high regularity hypothesis, might be possible using for example the Krein spaces framework \cite{Ong}. Such investigations are out of the scope of this paper and they are leaved for future work.  

\section{Eigenvalue decay revisited}\label{Hypo}
The fact that a given kernel $W$ satisfy any of the hypothesis $\text{H}_1$, $\text{H}_2$ and $\H_3$, of Theorem \ref{thm:theo2}, is not necessarily easy to verify. If an eigenfunctions basis is known, such as the case of spherically symmetric kernels treated in Section \ref{RG} below, the eigenvalues can be obtained by the computing the integral of the product of the kernel with the eigenfunctions of the basis. There is no guarantee that an analytic close solution exist in general, but in practice this procedure can be done numerically. On the other hand, when the eigenfunctions of the kernel are not known, we are left to solve often complicated differential equations. For that reason is useful to have an equivalent notion of regularity at hand.  

The decrease in the eigenvalues appears naturally as regularity hypothesis of the Sobolev-type. Indeed, given a measurable metric space $(\mathcal{X},\kappa,\nu)$ where $\kappa$ is a distance and $\nu$ a probability measure, we suppose that $\{\varphi_k\}_{k\in\mathcal{J}}$ is an orthonormal basis of $L^2(\mathcal{X},\nu)$, where $\mathcal{J}$ is a countable set. We define the weighted Sobolev space $S_{\omega}$ with associated positive weights $\omega=\{\omega_j\}_{j\in\mathcal{J}}$ as  
\[S_{\omega}(\mathcal{X}):=\Big\{f\overset{L^2}{=}\sum_{k\in\mathcal{J}}\hat{f}(k)\varphi_k\, \text{ s.t }\|f\|^2_{\omega}:=\sum_{k\in\mathcal{J}}\frac{|\hat{f}(k)|^2}{\omega_k}<\infty\Big\}\]
We can take, as in Section \ref{s:sect2}, the measurable metric space $(\Omega,\rho,\mu)$ and consider $\mathcal{X}=\Omega^2$ and $\nu=\mu\times\mu$. If $\phi_k$ is a basis of $L^2(\Omega,\mu)$ then a basis for $L^2(\mathcal{X},\mu)$ is given by $\{\varphi_k\}_{k,l}$ where $\varphi_{k,l}=\phi_k\otimes\phi_l$. Observe that here $\mathcal{J}=\{(k,l)\}_{k,l\in\N}$. We note that for a kernel $W$ in $S_{\omega}(\mathcal{X})$ with eigenvalues $\lambda_k$ and eigenvectors $\phi_k$, we have $\hat{f}(k,l)=\lambda_k\delta_{kl}$. If we want that the series in definition of the Sobolev space to converge, it is sufficient that $\lambda_k^2\frac{1}{\omega_k}=\O\big(\frac{1}{k^{1+\delta^\prime}}\big)$ where $\delta^\prime>0$. This allow to control the decay behavior of $\lambda_k$ by direct comparison to $\omega_k$.

When $\Omega$ is an open subset of $\mathbb{R}^d$, the classical definition of weighted Sobolev spaces makes use of the (weak)-derivatives of a function. If $\varrho:\Omega\rightarrow [0,\infty)$ is a locally integrable function, we define the weighted Sobolev space $\W^p_2(\Omega,\varrho)$ as the normed space of locally integrable functions $f:\Omega\rightarrow \mathbb{R}$ with $p$ weak derivatives such as the following norm is finite \[\|f\|_{p,\varrho}=(\int_{\Omega}|f(x)|^2d\varrho(x))^{\frac{1}{2}}+(\sum_{|\alpha|=p}|D^\alpha f(x)|^2d\varrho(x))^{\frac{1}{2}}\] where $\alpha$ is a multiindex and $D^\alpha$ are the weak derivatives. 

For a symmetric kernel $K:\mathbb{R}^d\times \mathbb{R}^d\rightarrow \mathbb{R}$ we can define the Sobolev regularity by the canonical embedding of $\mathbb{R}^d\times \mathbb{R}^d$ into $\mathbb{R}^{2d}$, but it seems more natural (see \cite[sect.~2.2]{Xu}) to say that the kernel satisfies the weighted Sobolev condition if $K(\cdot,x)\in \W^p_2(\Omega,\varrho)$ for all $x\in\Omega$. 
However, in some cases as in the  \emph{dot product kernels}, where there exists a real function $f:\mathbb{R}\rightarrow[0,1]$ such as $K(x,y)=f(\langle x, y \rangle)$, it is even more natural to say that $K$ that satisfies the Sobolev condition with weight $\varrho:\mathbb{R}\rightarrow \mathbb{R}$ if $f\in \W^p_2(\mathbb{R},\varrho)$. Intuitively speaking, given that $f$ is defined on $\mathbb{R}$, it seems natural to carry out the analysis in one dimension. 

In \cite{Nica} is proved that in the one dimensional case, both definitions of weighted Sobolev spaces are coincident. Otherwise stated, the following equality between metric spaces holds $S_{w}([-1,1])=\W^2_2([-1,1],\varrho_\gamma)$ where $\omega=\{\omega_k\}_{k\in \mathbb{N}}=\frac{1}{1+\nu_k}$, with $\nu_k=k(k+d-1)$ and $\varrho_\gamma(x)=(1-x^2)^{\frac{d-3}{2}}$. Here we recognize in $\nu_k$ the sequence of eigenvalues of the Laplace-Beltrami operator on $\S^{d-1}$ and $\varrho_\gamma$ is the weight that defines the orthogonality relations between the Gegenbauer polynomials $G^\gamma_{l}(\cdot)$ with $\gamma=\frac{d-2}{2}$. That means that two Gegenbauer polynomials of different degrees $G^\gamma_{k}, G^\gamma_{l}$ with $k\neq l$ are orthogonal in $L^2([-1,1],\varrho_\gamma)$, which is the space of square integrable functions defined in $[-1,1]$ with the weight $\varrho_\gamma$. We denote $\|\cdot\|_{2,\gamma}$ the norm in $L^2([-1,1],\varrho_\gamma)$, that is $\|f\|^2_{2,\gamma}=\int f^2(t)\varrho_\gamma(t)dt$. In the next section, we explore this case in more detail and highlight the connection with random geometric graphs. 

\section{Dot product kernels}\label{RG}
In this section we will consider the space $\Omega=\mathbb{S}^{d-1}$, with $d\geq 3$, equipped with $\rho$ the geodesic distance and the measure $\sigma$, which is the surface (or uniform) measure normalized to be a probability measure. Let $f:[-1,1]\rightarrow [0,1]$ be a measurable function of the form $W(x,y)~=~f(\cos{\rho(x,y)})$. Note that the geodesic distance on the sphere is codified by the inner product, that is $\rho(x,y)=\arccos\langle x,y \rangle$. Thus we directly assume, here and thereafter, that $W$ only depends on the inner product, that is \[W(x,y)=f(\langle x,y \rangle)\]
This family of kernels are usually known as \emph{dot product kernels} and they are rotation invariant, that is $W(x,y)=W(Ax,Ay)$ for any rotation matrix $A$, and its associated integral operator $T_W$ is a convolution operator. This type of kernel has been used in the context of random geometric graphs \cite{Yohann} and deep learning \cite{Cao}, to name a few. Similar to the context of Fourier analysis of one dimensional periodic functions, in this case we have a fixed Hilbertian basis of eigenvectors that only depends on the space $\Omega$, but not on the particular  choice of kernel $W$. The aforementioned basis is composed by the well-known \emph{spherical harmonics} \cite[chap.~1]{Dai}, which play the role of the Fourier basis in this case. For each $l\in \mathbb{N}$ we have an associated eigenspace $\mathcal{Y}_l$, known as the space of spherical harmonics of order $l$. Let $\{Y_{jl}\}^{d_l}_{j=1}$  be an orthonormal basis of $\mathcal{Y}_l$ and define $d_l=\operatorname{dim}(\mathcal{Y}_l)$, then by \cite[cor.~1.1.4]{Dai}\begin{equation}\label{eq:dlformule} d_l=\binom{l+d-1}{l}-\binom{l+d-3}{l-2}=\O(l^{d-2})\end{equation} for $l\geq 2$ and $d_0=1$,$d_1=d$. The second equality follows easily from the definition (see Appendix \ref{dl}). We define $\lambda^\ast_l$, the eigenvalue of $T_W$ associated with the corresponding space $\mathcal{Y}_l$. We use the $\ast$ subcript to difference this indexation(which follows the spherical harmonics order) from the decreasing order indexation $\{\lambda_i\}_{i\geq 1}$. As sets $\{\lambda^\ast_i\}_{i\geq0}$ and $\{\lambda_i\}_{i\geq 1}$ are equal (have the same elements), but in $\{\lambda^\ast_i\}$ the eigenvalues are counted without multiplicity (except if $\lambda$ is associated to more than one $\mathcal{Y}_l$ \footnote{In which case appears repeated a number of times equals to the number of $\mathcal{Y}_l$'s to which is associated.}). This seems more natural in this case, but have to keep this in mind when applying Theorems \ref{thm:theo1} and \ref{thm:theo2}. In this setting, the expansion \eqref{expan} becomes 
\begin{equation}\label{eq:expansion}
f(\langle x, y \rangle)=\sum_{l\geq0}\lambda^\ast_l\sum^ {d_l}_{j=0}Y_{jl}(x)Y_{jl}(y)
\end{equation}
On the other hand, the \emph{addition theorem} for spherical harmonics \cite[eq.~1.2.8]{Dai} gives
\begin{equation}\label{eq:addition}
Z_l(x,y)=\sum^ {d_l}_{j=0}Y_{jl}(x)Y_{jl}(y)\end{equation}
and the preceding equality does not depend on the particular choice of basis $\{Y_{jl}\}^{d_l}_{j=1}$. The $Z_l$ are called the \emph{zonal harmonics}. So based on \eqref{eq:expansion} we have the following 
\begin{equation}\label{eq:expansion2}
f(\langle x, y \rangle)=\sum_{l\geq0}\lambda^\ast_lZ_l(x,y)
\end{equation}
An important property is that each zonal harmonic $Z_l(x,y)$ is a multiple of the Gegenbauer (ultraspherical) polynomial of level $l$, hence it only depends on the inner product of $x,y\in \mathbb{S}^{d-1}$. The following classic result in harmonic analysis \cite[Thm.1.2.6, Cor. 1.2.7]{Dai} makes the previous statement more precise
\begin{proposition}\label{additionMaxim}
For any $x,y\in\mathbb{S}^{d-1}$, $l\in\mathbb{N}$, $d\geq 3$ and $\gamma=\frac{d-2}{2}$
\[Z_l(x,y)=c_lG^\gamma_l(\langle x,y \rangle)=c_l\sqrt{d_l}\tilde{G}^\gamma_l(\langle x,y \rangle)\]
where $c_l:=\frac{l+\gamma}{\gamma}=\frac{2l+d-2}{d-2}$, $G^\gamma_l$ is the $l$-th Gegenbauer (ultraspherical) polynomial and $\tilde{G}^\gamma_l~=~G^\gamma_l/\|G^\gamma_l\|_{2,\gamma}$. Furthermore, for any $l\in\mathbb{N}$, $Z_l$ attains its maximum in the diagonal, that is \[\max_{x,y\in\mathbb{S}^{d-1}}|Z_l(x,y)|=|Z_l(x,x)|=d_l\]
\end{proposition}
\begin{remark}
From Proposition \ref{additionMaxim} we derive a simple formula to compute the eigenvalues, using the orthogonality relations between Gegenbauer polynomials. We recall that given $\varrho_\gamma(x)=(1-x)^{\gamma}$ (the Sobolev weight defined in Section \ref{Hypo}) we have \[\int_{\S^d}\tilde{G}_k^\gamma(t)\tilde{G}_l^\gamma(t)\varrho_\gamma(t)dt=\delta_{kl}\] Defining $b_d=\frac{\Gamma(\frac{d}{2})}{\sqrt{\pi}\Gamma(\frac{d-1}{2})}$ we have
\begin{equation}
\label{lambdas}
\lambda^\ast_l=\Big(\frac{c_lb_d}{ d_l}\Big)\int_{-1}^1f(t)G^\gamma_l(t)\varrho_\gamma(t)dt=\frac{\Gamma(\frac{d}{2})}{\sqrt{\pi}\Gamma(\frac{d-1}{2})}\frac{l!}{(2d-2)^{(l)}}\int_{-1}^1f(t)G^\gamma_l(t)\varrho_\gamma(t)dt
\end{equation}
where $(a)^{(i)}=a\cdot(a+1)\cdots(a+i-1)$ is the \textit{rising factorial} or (rising) Pochammer symbol. 
\end{remark}
What precedes means that in this framework, the growth rate of the eigenvector is known and fixed, and the fulfillment of the hypotheses of Theorems \ref{thm:theo1} and \ref{thm:theo2} depends on the eigenvalue decay rate only, which can be verified using formula \eqref{lambdas} above. We exhibit explicit calculations for kernels related to geometric graphs in Section \ref{examples} below.

Given Proposition \ref{additionMaxim}, the hypothesis $\text{H}$ is implied by the following \[\sum_{l\geq 0}|\lambda^\ast_l| d_l<\infty \]
Because of the explicit value of $d_l$ (given in \eqref{eq:dlformule}), we get that $|\lambda^\ast_l|=O(l^{1-d-\varepsilon})$, for any $\varepsilon>0$, it is sufficient for $\H$ to hold.  The following lemma is a consequence of the Addition Theorem eq.\eqref{eq:addition}
\begin{lemma}\label{lem:varproxy1}
For any $i\in\N$ we have that \[\V_1(i)=\O(i)\] Consequently, for any $W$ such that $\lambda_i=O(i^{-\delta})$ with $\delta>1$ hypothesis $\H_1$ is satisfied. If $\lambda_i=O(e^{-\delta i})$ with $\delta>0$, then hypotheses $\H_2$ and $\H_3$ are satisfied. 

\end{lemma}

The following lemma allow us to relax the hypotheses of Thm.\ref{thm:theo2} and to obtain sharper results in this case.

\begin{lemma}\label{lem:equiv}
Let $W$ be a kernel such that $\V_1(i)=\mathcal{O}(i)$ for all $i$ and $\V_2(R)~=~\mathcal{O}(\sum_{i>R}|\lambda_i|)$, then the results of Theorem \ref{thm:theo2} are valid with $s=0$. 
\end{lemma}

\begin{remark}
Observe that in the previous lemma, the eigenvectors growth do not depend on $d$ in the indexation induced by $\{\lambda_i\}_{i\geq 1}$. However, given that the eigenvalues satisfy $\sum_{l\geq 0}|\lambda^\ast_l|d_l\leq \infty$ and that $d_l=\O(l^{d-1})$, the effect of the dimension on the hypothesis $\H_1$ and $\H_2$ is that is for a kernel on space $\R^d$ with high $d$ the condition $|\lambda_i|=\O(i^{-\delta})$ will be more restrictive, for any $\delta>0$, compared to a kernel defined on a lower dimensional space, simply because of the restrictions in the multiplicity. 
\end{remark}

\begin{cor}\label{cor:dotprod}
Let $W(x,y)=f(\langle x,y \rangle)$ be a dot product kernel in $\S^{d-1}$. Suppose that $f$ is in the Sobolev space $\W^{p}_2([-1,1],\varrho^\prime)$ where $\varrho^\prime(t)=(1-t^2)^{\frac{d-3}{2}}$.Then there exists $\varepsilon>0$ such that for $\alpha\in(0,1)$ we have with probability larger than $1-\alpha$, for $1\leq i\leq n$

\begin{equation}\label{eq:rate_dotprod}
|\lambda_i(T_n)-\lambda_i|\lesssim_\alpha i^{-\delta+1/2}n^{-1/2}
\end{equation}
with $\delta=\frac{p+\varepsilon}{d-1}+\frac{1}{2}$. 

\end{cor}
\begin{remark}
A similar framework to the one presented in this section was studied in \cite{Yohann} in the context of graphon estimation with a spectral algorithm. They bound the $\delta_2(\cdot,\cdot)$ metric, which implies using Weyl's inequalities a rate $\O_\alpha(n^{-\frac\delta{2\delta+d-1}})$ for $|\lambda_i(T_n)-\lambda_i|$. This bound is slower than the rate in \eqref{eq:rate_dotprod}.
\end{remark}

\begin{remark}
While we do not treat the case $d=2$, which corresponds to kernels on the circle $\S^1$, it is easy to replicate the analysis of this section by considering the usual Fourier basis, which is associated to integral operators on $\S^1$ .
\end{remark}

\subsection{Connections with random graphs}\label{s:sect2}

The generative model for dense random graphs known as $W$-\emph{random graph} model is based on notion of graphons, which are kernels bounded between $0$ and $1$. More specifically, given $W:\Omega^2\rightarrow [0,1]$  a square integrable kernel, we construct the graph with $n$ nodes defined by the $n\times n$ adjacency matrix $A$, where the entries $(A_{ij})_{1\leq i<j\leq n}$ are independent Bernoulli random variables with corresponding mean $W_{ij}$. We recall that $W_{ij}$ is the kernel matrix (which in this context has also been called \emph{probability matrix} \cite{Verzelen}). Even if the theory of graphons can be formulated, without loss and generality, in the case $\Omega=[0,1]$, this is not always convenient. Sometimes the representation on the interval is less revealing \cite[p.~190]{Lov}. Later, in the geometric graph setting we will make an explicit use of the sphere geometry throughout the geodesic distance.
 
Note that by choosing $W(x,y)=p$ we obtain the well-known Erd\"os-R\'enyi model. On the other hand, if $\{U_i\}_{1\leq i\leq k}$ is a partition of $\Omega$ and $W(x,y)=p_{ij}$, for $p_{ij}\in[0,1]$, when $x\in U_i$ and $y\in U_j$, we obtain the Stochastic Block Model with $k$ groups (see \cite{Holland}). If we consider a metric $\rho$ on $\Omega$ and the graphon $W(x,y)=\mathbbm{1}_{\rho(x,y)<\tau}$ we recover the classical random geometric graph of threshold parameter $\tau$. A comprehensive probabilistic study of the latter model can be found in \cite{Pen}. Given the success of spectral methods in many graph related algorithms, it comes as no surprise that good understanding of the spectrum of kernel (which is related to the adjacency matrix) will shed light on the properties of the graphs produced by this generative model. The study of how the analytical properties of $W$ are related to the combinatorial properties of the graph $G$ is the matter of the excellent  manuscript \cite{Lov}.

\section{Illustrative examples}\label{examples}

We will consider a series of examples of dot product kernels on the sphere, which are key in the definition of non-parametric (or soft) random geometric graphs \cite{Yohann}. We start with a less explicit kernel, only to have a better picture of how the bounds in Theorem \ref{thm:theo2} translate into numbers.

\subsection{Kernels satisfying $\H_1$, $\H_2$ and $\H_3$}

We first consider a kernel $W$ defined on a probability space $(\Omega,\mu)$ which satisfy $\H_1$ with $s=0$. The case $s=0$ is satisfied by many widely the kernels, such as the dot product kernels on the sphere and some Gaussian kernels (see Sections \ref{sec:ex_dpk} and \ref{sec:ex_gau} below). Recall that, when knowledge about the parameter $s$ is available, constructing kernels that satisfy the regularity hypothesis amounts to choose the right eigenvalue sequence, using \eqref{expan}. We will apply Theorem \ref{thm:theo2} for different values of $\delta$ and for different indices across the spectrum. For that we will introduce the parameter $\beta$ such that $i=n^\beta$. We discretize $(0,1)$ with a step of $0.1$ and tabulate the obtained rate for the different $\beta$'s. What we mean here by the rate is the value $n^{h}$, where $|\lambda_i(T_n)-\lambda_i|\lesssim_\alpha n^{h}$, with probability larger than $1-\alpha$. We summarize the rates in Table \ref{tables_0} below.

\begin{table}[h!]
\centering
\begin{tabular}{l|*{12}{c}}
\backslashbox{$h$}{$\beta$} & $0$& $0.1$& $0.2$& $0.3$& $0.4$ & $0.5$ & $0.6$ & $0.7$ &$0.8$ & $0.9$ \\ \hline
$\delta=4$&$-0.5$ & $-0.85$& $-1.2$&  $-1.55$& $-1.9$&  $-2.25$& $-2.6$&  $-2.95$ &$-3.3$ & $-3.65$\\ \hline
$\delta=5$&$-0.5$ &$-0.95$& $-1.4$ & $-1.85$& $-2.3$ & $-2.75$ &$-3.2$ & $-3.65$ &$-4.1$&  $-4.55$\\ \hline
$\delta=6$&$-0.5$ &$-1.05$& $-1.6$&  $-2.15$& $-2.7$ & $-3.25$ &$-3.8$&  $-4.35$& $-4.9$ & $-5.45$\\ \hline
$\delta=7$&$-0.5$ & $-1.15$& $-1.8$&  $-2.45$ &$-3.1$ & $-3.75$& $-4.4$ & $-5.05$& $-5.7$ & $-6.35$\\ \hline
$\delta=8$&$-0.5$ & $-1.25$ &$-2$ &  $-2.75$& $-3.5$  &$-4.25$& $-5$  & $-5.75$ &$-6.5$  &$-7.25$
\end{tabular}
\caption{Rates for a dot product kernel satisfying $\H_1$ for $s=0$ and different regularity parameters.}
\label{tables_0}
\end{table}

We now repeat this exercise in the case of a kernel in $\Omega$ satisfying $\text{H}_1$ with $s=1$, which is useful when considering the a kernel such that its associated integral operator has as eigenfunctions the elements of a Hilbertian basis of $L^2([0,1])$. It is known that the classic Legendre polynomials, that we denote $\{p_k\}_{k\geq 0}$, provide a basis in that case and that $\|p_k\|_{\infty}\propto \sqrt{k}$ \cite[Chap.8]{Sze}. 

  \begin{table}[h!]
\centering
\begin{tabular}{l|*{12}{c}}
\backslashbox{$h$}{$\beta$} & $0$& $0.1$& $0.2$& $0.3$& $0.4$ & $0.5$ & $0.6$ & $0.7$ &$0.8$ & $0.9$ \\ \hline
$\delta=4$&$-0.5$ & $-0.88 $& $ -1.22$& $-1.06$&$        -1.25 $&       $-1.43$&$-1.7 $&  $    -1.9$&$ -2.1 $&$ -2.3$ \\ \hline
$\delta=5$&$-0.5$ &    $ -0.98$& $   -1.425  $& $    -1.34$& $        -1.62$& $        -1.9$& $
  -2.3$& $        -2.6$& $        -2.9$& $        -3.2 $\\ \hline
$\delta=6$&$-0.5$ & $    -1.08 $& $    -1.62  $& $   -1.4  $& $      -1.7$& $        -2$& $
  -2.3 $& $       -2.6$& $        -2.9  $& $      -3.2  $ \\ \hline
$\delta=7$&$-0.5$ &  $    -1.18 $& $-1.83$& $     -1.91$& $        -2.38 $& $        -2.85$& $         -3.5$& $ 
  -4$& $        -4.5$& $         -5 $\\ \hline
$\delta=8$&$-0.5$ &  $  -1.28  $& $-2.03$& $ -2.2   $& $     -2.77  $& $       -3.34$& $
  -4.1   $& $      -4.7   $& $     -5.3.5    $&  $     -5.9  $ \end{tabular}
\caption{Rates for a dot product kernel satisfying $\H_1$ for $s=1$.}
\label{tabledot}
\end{table}



In the case of hypothesis $\H_2$, the rate takes an exponential form. For example, taking $i=n^\beta$, we obtain exponential rates of the form $e^{-\delta n^\beta+(\beta(s+\frac12)-\frac12)\log n}$, which are similar to those obtained by approximation theoretic methods in \cite{Belkin}. Indeed, the positive radial basis differentiable kernels in \cite{Belkin}, provide examples of kernels satisfying either $\H_2$ or $\H_3$. 

\subsection{Dot product kernels and random geometric graphs}\label{sec:ex_dpk}

We consider examples defined on $\Omega=\mathbb{S}^{d-1}$ equipped with the normalized uniform measure on the sphere $\sigma$ and a graphon $W:\mathbb{S}^{d-1}\times\mathbb{S}^{d-1}\rightarrow [0,1]$. In addition, we will consider that $W(x,y)=f(\langle x,y\rangle)$ where $f:[-1,1]\rightarrow [0,1]$ is a measurable function. A common fact for those models is that the (normalized) \text{degree function}, as defined in \cite[p.~116]{Lov}, is constant. Indeed\[d_K(x)=\int_{\mathbb{S}^{d-1}}f(\langle x,y\rangle)d\sigma(y)=\int_{\mathbb{S}^{d-1}}f(\langle x^\prime,y\rangle)d\sigma(y)=d_K(x^\prime)\]
for all $x,y\in\mathbb{S}^{d-1}$, which follows directly from the rotational invariance of the inner product and a change of variables. This readily implies that $W$ is a canonical kernel.  Taking the function $g:\mathbb{S}^{d-1}\rightarrow\mathbb{S}^{d-1}$ with $g(x)=1$ for all $x\in\mathbb{S}^{d-1}$, we see that $d_K$ is an eigenvalue associated with the function $g$. Moreover, it is not hard to see that $d_K$ will be the largest eigenvalue of $T_W$, see \cite[Supp. Mat.]{Ara}. 

\subsubsection{Constant and linear graphons} Since smooth graphons on the sphere can be conveniently approximated by series of Gegenbauer (ultraspherical) polynomials, we describe here what their spectrum looks like in the finite rank case.

Likewise to Section \ref{RG}, we consider $\gamma=\frac{d-2}{2}$. We start with the \emph{constant graphon} with is the related to the first polynomial in the Gegenbauer basis which is $G^\gamma_0(t)=1$. More specifically, we consider $W_1(x,y)=p_0G^\gamma_0(\langle x,y\rangle)=p_0$, where $p_0\in [0,1]$, which is a rank $1$ graphon. This coincides with the well-known \textit{Erd\"os-R\'enyi} graphon. If we generate a graph with this model, following Section \ref{s:sect2}, with  $\{X_i\}_{1\leq i\leq n}$ a uniform sample on the sphere, then the probability that $X_i$ and $X_j$ are connected for any $i,j\in\{1,\cdots,n\}$ is $p_0$. That is, for any two nodes the probability that they are connected is the same, regardless of its position on the sphere. For this reason, this model is considered as structureless (see \cite{Bub}). Its eigenvalues are (we use the decreasing indexing)
\[ \begin{cases}
        {\lambda_1}^{(1)}&=p_0\\
        {\lambda_i}^{(1)}&=0, \text{for all }i>1
\end{cases}
\]
In this case, the eigenvalue ${\lambda_1}^{(1)}$ which has multiplicity one has a clear interpretation in the context of graphon theory. Indeed, if we consider the normalized degree we have 
\[d_K(x)=\int_{\mathbb{S}^{d-1}}p_0d\sigma(y)=p_0\]
thus the non-zero eigenvalue  $\lambda^{(1)}_1$ is just the mean degree of the graphon (which asymptotically will be the mean degree of the generated graph). We note $T^{(1)}_n$ the kernel matrix associated with $W_1$. Observe that in this case we can apply Theorem \ref{thm:theo1}, for instance, and given the multiplicity of $\lambda_1$ we obtain
\[|\lambda_1(T^{(1)}_n)-{\lambda_1}^{(1)}|\lesssim \frac{p_0}{\sqrt{n}}\sqrt{\log{d/\alpha}}\] 
with probability larger than $1-\alpha$. For all $i> 1$ we have $\lambda_i(T^{(1)}_n)=0$.
We now consider the graphon $W_2(x,y)=p_0G^\gamma_0(\langle x,y\rangle)+p_1G^\gamma_1(\langle x,y \rangle)=p_0+p_12\gamma \langle x,y\rangle$, which has rank $1+d_1$ , where the $d_1$ is the first spherical harmonic space dimension given in \eqref{eq:dlformule}. It is easy to see that $d_1=d$. This graphon is based on the first two Gegenbauer (ultraspherical) polynomials $G^\gamma_0(t)=1$ and $G^\gamma_1(t)=2\gamma t$ and we call it \textit{linear graphon}.  The eigenvalues for this model are given by
\[ \begin{cases}
        {\lambda_1}^{(2)}&=p_0\\
        {\lambda_2}^{(2)}&=p_1\frac{d-2}{d}\\
        {\lambda_i}^{(2)}&=0, \text{for all }i>2
  \end{cases}\]
The eigenvalue ${\lambda_1}^{(2)}$ has multiplicity one and the eigenvalue ${\lambda_2}^{(2)}$ has multiplicity $d$.
At first glance ${\lambda_2}^{(2)}$ is $O(1)$, but since by definition a graphon takes values $0\leq W_2(x,y)\leq 1$, the values $p_0$ and $p_1$ must satisfy certain constraints. In this particular case, we see that $p_0\in[0,1]$ and $p_0\pm p_12\gamma\geq 0$ so $|p_1|\leq \frac{p_0}{2\gamma}$. That implies that ${\lambda_2}^{(2)}$ is decreasing on $d$. More specifically, since $\gamma=\frac{d-2}{2}$ we have ${\lambda_2}^{(2)}=O(\frac{1}{d})$. As we saw in Section \ref{RG}, here we have $\V(d+1)=d+1$. 
Applying Theorem \ref{thm:theo1} we obtain, for $\alpha\in (0,1)$ and with probability bigger than $1-\alpha$. 
\begin{align*}
|\lambda_1(T^{(2)}_n)-{\lambda_1}^{(2)}|&\lesssim \frac{p_0}{\sqrt{n}}\sqrt{(d+1)\log{d/\alpha}}\\
|\lambda_2(T^{(2)}_n)-{\lambda_2}^{(2)}|&\lesssim \sqrt{\frac{\log{d/\alpha}}{nd}}
\end{align*}
 For all $i>2$ we have $\lambda_i(T^{(2)}_n)=0$.

 Here we see clearly how the relative concentration inequality improves the accuracy with respect to a simple application of Weyl-type inequalities. More specifically, for the eigenvalue ${\lambda_1}^{(2)}$ we get a better dimensional dependence. If we apply Weyl inequality, we will obtain for ${\lambda_1}^{(2)}$ the same order of concentration that for ${\lambda_0}^{(2)}$, that is $\O(\frac{1}{\sqrt{n}}\sqrt{2(d+1)\log{d/\alpha}})$. Using Theorem \ref{thm:theo1} instead we get $\O(\frac{1 }{\sqrt{nd}}\sqrt{\log{d/\alpha}})$, which is much better.

As a side note, we see that as the dimension $d$ increases, the eigenvalue ${\lambda_2}^{(2)}$ tends to $0$. Using Weyl inequality gives a bound that deteriorates in the dimension, such as the one for $\lambda_1(T^{(2)}_n)$ above. On the other hand, the relative concentration bound perform better having a scaling term that decrease with the dimension (and in the other term the dimension increase but only logarithmically), giving a much better picture of what actually happens. 
\subsubsection{Proximity and logistic graphons} \label{sec:prox_log}

We consider the graphon $W_g(x,y)=\mathbf{1}_{\langle x,y\rangle \geq 0}$, which in the dense setting is equivalent to angular version of the classic random geometric graph \cite{Pen}( threshold parameter $\tau=0$). If we generate a random graph by the model described in Section \ref{s:sect2}, with $\{X_i\}_{i\in\mathbb{N}}$ a uniform sample on the sphere, then the corresponding nodes $X_i$ and $X_j$ will be connected if and only if they belong to the same semi-sphere. Applying \eqref{lambdas} we get \[{\lambda^\ast_l}=a_{l,d}\int^1_{0}G^\gamma_l(t)\varrho_\gamma(t)dt=a_{l,d}b_{l,d}\int^1_{0}\frac{d^l}{dt^l}w_{\gamma+l}(t)dt\]
where $a_{l,d}=\frac{c_lb_d}{d_l}$ and $b_{l,d}=\frac{(-1)^l}{2^ll!}\frac{(2d-2)^{(l)}}{\big(\frac{d-1}{2}\big)^{(l)}}$. In the second equality we have used the Rodrigues formula for the Gegenbauer polynomials. The details of this and all the computations of this section are in Appendix \ref{compu}. It is easier to describe the eigenvalues following the spherical harmonics order, using the $\ast$ notation
\[ \begin{cases}
{\lambda^\ast_0}&=\frac{1}{2}\\
{\lambda^\ast_i}&=0\text{, for }i>0\text{ even}\\
{\lambda^\ast_i}&=\frac{(-1)^{l+\ceil{l/2}}}{2\pi }Beta\big(\frac{d}{2},\frac{l}{2}\big)\text{, for }i\text{ odd}
\end{cases}\]
where $Beta(x,y)=\frac{\Gamma(x)\Gamma(y)}{\Gamma(x+y)}$ is the Beta function. Since this function is neither regular nor finite rank we cannot apply directly Theorems \ref{thm:theo1} and \ref{thm:theo2}. Indeed, from the expression for eigenvalues ${\lambda^\ast_l}$ we deduce that for $d$ fixed, asymptotically as $l$ tends to infinity \[|{\lambda^\ast_l}|\sim \Gamma\big(\frac{d}{2}\big)\cdot\big(\frac{l}{2}\big)^{-\frac{d}{2}}\]
using the Stirling asymptotic approximation of the Beta function. 

Clearly the eigenvalues do not fulfill hypothesis $\H$. Indeed, by \eqref{eq:dlformule} the series with term $|\lambda^\ast_l|d_l$ is not summable. Nevertheless, we can apply the results to the $m$-fold composition of the operator $T_{W_g}^{\circ m}$ with $m\in\mathbb{N}$, which is an integral operator with kernel:\[W_g^{\circ m}(x,y)=\int_{(\mathbb{S}^{d-1})^{m-1}}W_g(x,z_1)\cdot W_g(z_1,z_2)\cdots W_g(z_{m-1},y)d\sigma(z_1)\cdots d\sigma(z_{m-1})\]
where $(\mathbb{S}^{d-1})^{m-1}$ is the $m-1$ product space of the $d$-dimensional unit sphere. 

It is well known that $\lambda(T^{\circ m}_{W_g})=\{\lambda_i^m\}_{i\geq 1}$. In other words, the following $L^2$ expansion holds\[W_g^{\circ m}(x,y)=\sum^\infty_{k=0}\lambda^m_k\phi_k(x)\phi_k(y)\] Taking $m\geq2$ and using the previous estimation, we have that \[|\lambda^\ast_l|^m\sim \Big(\frac{l}{2}\Big)^{-\frac{dm}{2}}\] This implies that $|\lambda_i|^m=\O(i^{\frac{dm}{2(d-1)}})$ (see Appendix \ref{proofcor1}). Thus, $W^{\circ m}_g$ satisfies the hypothesis $\H$ and $\H_1$. In the case $m=2$, the kernel matrix is \[(T^{\circ 2}_n)_{ij}:=\frac{1}{n}\int_{\mathbb{S}^{d-1}}W_g(X_i,z)W_g(z,X_j)d\sigma(z)\] and we have ${|\lambda_i|}^2=O(l^{-\frac{d}{d-1}})$ and $\V_1(i)=O(i)$ by Lemma \ref{lem:varproxy1}. The previous implies that $W^{\circ 2}_g\in \mathcal{W}^{p}_2([-1,1],\varrho^\prime)$ for $p=\frac12-\varepsilon$ for every $\varepsilon>0$. Using Cor. \ref{cor:dotprod} we obtain, with probability higher than $1-\alpha$ 
\[|\lambda_i^2-{\lambda_i(T^{\circ 2}_n)}^2|\lesssim_\alpha i^{-\delta+\frac{1}{2}}n^{-1/2} \]
with $\delta=\frac12 +\frac{1}{2(d-1)}$.

We consider the logistic graphon $W_{lg}(x,y)=f(\langle x,y\rangle)$, where $f(t):=\frac{e^{rt}}{1+e^{rt}}=\frac{1}{1+e^{-rt}}$. This model was introduced in \cite{Hoff} and since then many variants have appeared. The symmetry with respect to $\frac{1}{2}$ of the logistic function, implies by \eqref{lambdas} that the eigenvalues of $W_{lg}$ are given by
\[\lambda^\ast_l=a_{l,d}\int_{-1}^1f(t)G^\gamma_{l}(t)\varrho_\gamma(t)dt=a_{l,d}b_{l,d}\int^1_{0}\frac{1-e^{-rt}}{1+e^{-rt}}\frac{d^l}{dt^l}w_{\gamma+l}(t)dt\]
The eigenvalues of $T_{W_{lg}}$ depend on $r$ in such a way that when $r= 0$ the spectrum of $T_{W_{lg}}$ coincide with the spectrum of the constant graphon with parameter $p_0=1/2$ and when $r\to \infty$ the spectrum of $T_{W_{lg}}$ converge to the spectrum of $T_{W_g}$. We can regard the logistic graphon model as an interpolation between the constant (Erd\"os-R\'enyi) graphon and the proximity (geometric) graphon. It is interesting to note that when $r=0$ the rank of $W_{lg}$ is one and when $r>0$, $W_{lg}$ has infinite rank. It is easy to see that for $r>0$ we will have roughly the same problem that in the previous case, as the eigenvalues will not satisfy the asymptotic decay conditions in the definition $\H$. This is, again, a manifestation of the fact the operator associated to $W_{lg}$ is Hilbert-Schmidt, but not trace-class. Using the square operator $T^{\circ 2}_{W_{lg}}$ we obtain a similar result that in the previous case for $r>0$ for the eigenvalues, which results in a slower rate.

\subsection{Gaussian kernel}\label{sec:ex_gau}
First we consider a kernel used in the context of Gaussian regression \cite{Zhu2}. In the one dimensional version, we take $\Omega=\mathbb{R}$ with a measure $\mu$ with density with respect to the Lebesgue measure $d\mu(x)=\frac{1}{\sqrt{\pi}}e^{-x^2}dx$ and $K(x,y)~=~e^{-\frac{1}{2}x^2-\frac{1}{4}(x-y)^2-\frac{1}{2}y^2}$. Its eigenvalues and normalized eigenfunctions are given in \cite{Zhu2}[sec.4] (see also \cite[sec.~6.2]{Fass}), which in the unidimensional case are, for $k\in\N$
\[\lambda_k=\frac{2^{-2k}}{(\frac{1}{2}(1+\sqrt{2})+\frac{1}{4})^{k+\frac{1}{2}}}\leq \frac{2}{5^{k+\frac{1}{2}}}\]
\[\phi_k(x)=\frac{\sqrt[\leftroot{-2}\uproot{2}8]{2}}{\sqrt{2^kk!}}\exp{\big(-\frac{x^2}{\sqrt{2}}\big)}H_k(\sqrt[\leftroot{-2}\uproot{2}4]{2}x)\]where $H_k(\cdot)$ is the $k$-th order Hermite polynomial (see \cite[Ch.5]{Sze}).  We note that the eigenvalues have an exponential decreasing rate. On the other hand, using the results in \cite{Indri} we have for all $x$ \[\exp{(-\frac{x^2}{\sqrt{2}})}H_k(\sqrt[\leftroot{-2}\uproot{2}4]{2}x)\leq \sqrt{2^kk!}\]  

Thus $\|\phi_k\|_\infty\leq \sqrt[\leftroot{-2}\uproot{2}8]{2}$. Consequently, the hypothesis $\text{H}_2$ for Theorem \ref{thm:theo2} holds with $s=0$ and $\delta=\log{5}$. We apply Theorem \ref{thm:theo2}, obtaining with probability larger than $1-\alpha$ 
\[|\lambda_i(T_K)-\lambda_i(T_n)|\lesssim_\alpha e^{-i\log{5}}n^{-1/2}\leq e^{-1.6 i}n^{-1/2}\]
where $T_n$ is the normalized kernel matrix. 

Now, we consider the kernel $K_2(x,y)=e^{-\frac{1}{4}(x-y)^2}$ with the same $\Omega$ and $\mu$. It is well known \cite[sec.4]{Zhu2}that the eigenvalues are the same as the case of $K$ above. The $L^2$ normalized eigenfunctions are \cite[sec.~6.2]{Fass}
\[\phi_k(x)=\frac{\sqrt[\leftroot{-2}\uproot{2}8]{2}}{\sqrt{2^kk!}}\exp{\big(-(\sqrt{2}-1)\frac{x^2}{2}\big)}H_k(\sqrt[\leftroot{-2}\uproot{2}4]{2}x)\]

Notice that also in this case the functions also have a uniform bound (which is larger than in the previous case) and the same result applies. That is 
\[|\lambda_i(T_{K_2})-\lambda_i(T_n)|\lesssim_\alpha e^{-1.6 i}n^{-1/2}\]
with probability larger than $1-\alpha$. 

\section{Conclusion}

We proved concentration inequalities that quantify the deviation of single eigenvalues of kernel matrices with respect to the corresponding eigenvalue of the kernel operator. Our inequalities are relative, in the sense that they scale with the eigenvalue in consideration, improving the accuracy compared to a single application of Weyl inequalities. This results in convergence rates that are often better than parametric $\O(n^{-1/2})$. We specialized our results, in Theorem \ref{thm:theo2}, to the case of regularity conditions with are common in the kernel literature, such as the polynomial or exponential decay of the eigenvalues. These hypotheses are complemented with assumptions on the growth of the eigenvectors, which guarantees the convergence of the spectral expansion of the kernel. We show that these hypotheses are well adapted to the case of dot product kernels, highlighting the relation with classical weak derivative definition of weighted Sobolev spaces. 

Compared to other results in the literature of kernel matrices concentration, Theorems~\ref{thm:theo1} and \ref{thm:theo2} have the advantage that they apply to the case of non-positive (or indefinite) kernels, while being faster, or on par, in terms of rates with previous results in \cite{Braun,Rudi, Belkin}, for example. This is specially important in the context of dense network analysis, where the kernel (graphon) is typically indefinite. This is illustrated by the proximity (or geometric) and the logistic graphons of examples in Section \ref{sec:prox_log}. 

Our approach is based on three steps: approximation, perturbation and concentration. There might be some improvements in the concentration step, where we use a somewhat rough bound for the operator norm of the residual matrix $E_R$. Standard matrix concentration tools do not seem to be adapted or do not give good results for a matrix of the form of $E_R$. An improvement on this front will have more impact on kernels satisfying $\H_1$ (in $\H_2$ and $\H_3$ the exponential decay of eigenvalues compensates for this overestimation of $\|E_R\|_{op}$). We show, in the proof of Thm.\ref{thm:theo2} with $\H_1(s=0)$ for example, how in certain cases we can achieve better results by a using a partition of the indices $[R,\infty]\cap \N$ and applying our concentration results (such as Proposition \ref{prop:Rlarger}) multiple times. Another route to achieve better results will using a comparison approach (find a majorizing process to $\|E_R\|_{op}$ which is easier to bound) \cite{VanHandel} might deliver tighter results, but finding such a majorizing process is in general a challenging task. 

While applying to a various situations, our hypothesis $\H$ might be improved to include kernels with lower regularity, such as those presented in Section \ref{sec:prox_log}. We overcome this by taking the $m$-fold composition of the kernel operator in consideration, which is detrimental to the rates. The main bottleneck is to work without the pointwise equality of the spectral expansion of the kernel, which is crucial in our approach. Finding ways to improve or remove this hypothesis is left for future work. 

Another interesting line of research is the extension of Hilbert space methods for kernel matrices (as those presented in \cite{Shawe,Blanchard,Ros} and specially the approximation theoretic approach in \cite{Belkin}) to the case of indefinite kernels. Those approaches usually work by considering the Nystrom extension of the kernel matrix (as is done in \cite{Ros}) and rely on RKHS methods. They have as a strong requisite the fact that the kernel is positive semidefinite. A possible way to proceed is to consider the Krein spaces technology, which has been previously used in the context of learning in \cite{Ong}.


\bibliographystyle{plain}
\bibliography{BiblioColt2.bib}

\appendix

\section{Proofs}

\subsection{Proof of Proposition \ref{prop:con_Gram}}\label{finiterankproof}

We note that \[\Phi^T_R\Phi_R=\sum^n_{j=1}Z_jZ^T_j\] where $Z_j\in\mathbb{R}^R$ is given by \[(Z_j)_k=\frac{1}{\sqrt{n}}\phi_k(X_j)\]
To obtain a tail bound we use matrix concentration results for sum of symmetric random matrices. More specifically, we will use a matrix Bernstein theorem, for the sum \[\Phi^T_R\Phi_R-\Id_R=\sum^n_{j=1}(Z_jZ^T_j-\frac{1}{n}\Id_R)\]It is a direct consequence of the definition that $\mathbb{E}[Z_jZ^T_j]=\frac{1}{n}\Id_R$. We note that \[\|Z_jZ^T_j-\frac1n\Id_R\|_{op}\leq |\lambda_{max}(Z_jZ^T_j)-\frac{1}{n}|\vee |\lambda_{min}(Z_jZ^T_j)-\frac{1}{n}|\] As $Z_jZ_j^T$ is positive definite, we have $\lambda_{max}(Z_jZ_j^T)\geq 0$ and $\lambda_{min}(Z_jZ_j^T)\geq 0$. Also \[\lambda_{min}(Z_jZ_j^T)\leq \lambda_{max}(Z_jZ_j^T)=\|Z_j\|^2=\sum^{R-1}_{k=0}\frac{1}{n}\phi^2_k(X_j)\] Given that $\sum^{R-1}_{k=0}\phi^2_k(X_j)\leq \V_1(R)$, we have \[\|Z_jZ^T_j-\frac1n\Id_R\|_{op}\leq \frac{|\V_1(R)-1|}{n} \] On the other hand
\begin{align*}
\mathbb{E}[(Z_jZ_j^T-\frac{1}{n}\Id_R)^2]&=\mathbb{E}[\|Z_j\|^2Z_jZ_j^T]-\frac{1}{n^2}\Id_R\\
\end{align*}
In addition, \[\E[\|Z_j\|^2Z_jZ_j^T]\preccurlyeq \frac{\V_1(R)}{n}\]
Consequently \[\|\mathbb{E}[(Z_jZ_j^T-\frac{1}{n}\Id_R)^2]|\|_{op}\leq\frac{|\V_1(R)-1|}{n^2} \]
then \[\|\sum^n_{k=1}\mathbb{E}[(Z_jZ_j^T-\frac{1}{n}\Id_R)^2]\|_{op}\leq \frac{|\V_1(R)-1|}{n}\]
Using the matrix Bernstein theorem with $S_j=Z_jZ_j^T-\frac{1}{n}\Id_R$, $d=R$, $L=\frac{|\V_1(R)-1|}{n}$ and $\sigma^2=L$ we get \begin{align*}\mathbb{P}(\|\Phi^T_R\Phi_R-\Id_R\|_{op}\geq t)&\leq R\exp{\frac{-t^2}{2(\sigma^2+\frac{Lt}{3})}}\\ 
&\leq R\exp{\frac{-nt^2}{2\V_1(R)(1+\frac{t}{3})}}
\end{align*}
This implies that for $\alpha\in (0,1)$ we have 
\begin{equation}\label{eq:almostort}
\|\Phi^T_R\Phi_R-\Id_R\|_{op}\leq  \sqrt{\frac{3\V_1(R)\log(2R/\alpha)}{n}}
\end{equation}
with probability bigger than $1-\alpha$. 


\subsection{Proof of Corollary \ref{cor:dotprod} }\label{proofcor1}
We will assume without loss of generality that $\{\lambda^\ast_k\}_{k\geq 0}$ is order decreasingly. Indeed, if $\sum_{l\geq 0}|\lambda^\ast_l| d_l<\infty$ holds, then $\{\lambda^\ast_k\}_{k\geq k_0}$ for some $k_0\in\N$ large enough, given that $d_l\asymp l^{d-2}$(with means that there exists $c,C>0$ such that $c l^{d-2}\leq d_l\leq C l^{d-2}$ for $l$ large enough). Define $l:\N\rightarrow\N$ to be the such that $\lambda_{i}=\lambda^\ast_{l(i)}$. 
From the relation $\sum^{l(i)-1}_{l=0}d_l\leq i\leq \sum^{l(i)}_{l=0}d_l$ we obtain \[(l(i)-1)^{d-1}\lesssim i\lesssim l(i)^{d-1}\] which implies that $l(i)=\O(i^{\frac{1}{d-1}})$.

Givent that $f\in \W^{p}_2([-1,1],\varrho^\prime)$ the eigenvalues $\lambda^\ast_l$ satisfy $\sum_{l\geq 0} |\lambda^\ast_l|^2 d_l(1+\nu_l^{p})<\infty$, where $\nu_l=l(l+d-1)$. This implies that $|\lambda^\ast_l|=\O(l^{-\delta^*})$ with $\delta^\ast=p+\frac{d-1}{2}+\varepsilon$ and $\varepsilon>0$. In consequence, we have $|\lambda_i|=\O(i^{-\delta})$, with $\delta:=\frac{p+\varepsilon}{d-1}+\frac{1}{2}$. By Lemma \ref{lem:varproxy1}, we have $\V_1(i)=\O(i)$. On the hand,  $\V_2(R)^{\frac12}=\|\sum_{i>R}\lambda_i \phi^2_i\|_\infty\lesssim\sum_{i>R}|\lambda_i|$ for all $R\in\N$.Then, by Lemma \ref{lem:equiv}, the rate is equivalent to using Thm. \ref{thm:theo2} with hypothesis $\text{H}_1$ with $s=0$ and $\lambda_i=\O(i^{-\delta})$. That is
 \[|\lambda_i(T_n)-\lambda_i|\lesssim_\alpha i^{-\delta+1/2}n^{-1/2}
\]
with probability larger than $1-\alpha$.





\subsection{Proof of Prop. \ref{prop:tail_bounds}}\label{A:proof5}


\subsubsection{Tail bound for $\sum_{l}\|E_R\phi_l\|$}

Fix $l\in\{1,\cdots,R\}$, then if $E_R\phi_l(i)$ is the $i$-th coordinate of $E_R\phi_l(i)$ for $1\leq i\leq n$, we have \[(E_R\phi_l(i))^2=\frac1{n^3}\sum_{1\leq j,j'\leq n}\sum_{k,k'>R}\lambda_k\lambda_{k'}\phi_k(i)\phi_k(j)\phi_l(j)\phi_{k'}(i)\phi_{k'}(j')\phi_l(j')\] hence 
\[\|E_R\phi_l\|^2=\frac{1}{n^3}\sum_{1\leq i, j,j'\leq n}\sum_{k,k'>R}\lambda_k\lambda_{k'}\phi_k(i)\phi_k(j)\phi_l(j)\phi_{k'}(i)\phi_{k'}(j')\phi_l(j')\]
which corresponds to a $U$-statistics of degree $3$ with kernel \[h(x,y,z):=\sum_{k,k'>R}\lambda_k\lambda_{k'}\phi_k(x)\phi_{k'}(x)\phi_k(y)\phi_l(y)\phi_{k'}(z)\phi_l(z)\]

We will with decoupled $U$-statistics. To pass from the decoupled to the undecoupled case is standard given the results in \cite{Delapena}(which affects the numerical constants). If $X,Y,Z$ are i.i.d random variables in $\Omega$, we have for all $x,y,z\in\Omega$
\begin{align*}
\E_X[h(X,y,z)]&=\sum_{k>R}\lambda^2_k\phi_k(y)\phi_l(y)\phi_{k}(z)\phi_l(z)\\
\E_Y[h(x,Y,z)]&=0\\
\E_Z[h(x,y,Z)]&=0\\
\E_{YZ}[h(x,Y,Z)]&=0\\
\E_{XZ}[h(X,y,Z)]&=0\\
\E_{XY}[h(X,y,Z)]&=0\\
\E_{XYZ}[h(X,Y,Z)]&=0
\end{align*}

where in the first three equalities we use the orthogonality of $\phi_i$ and $\phi_j$ for $i\neq j$. Define $h'(y,z):=\sum_{k>R}\lambda^2_k\phi_k(y)\phi_l(y)\phi_{k}(z)\phi_l(z)$. Given the previous computations is easy to see that $\tilde{h}(x,y,z)=h(x,y,z)-h'(y,z)$ is canonical. We have the following decomposition for the decoupled $U$-statistic of degree $3$ with kernel $h$, 
\begin{equation}
\frac1{n^3}\sum_{1\leq i_1,i_2,i_3\leq n}h(X^{(1)}_{i_1},X^{(2)}_{i_2},X^{(3)}_{i_3})=\frac{1}{n^3}\Big(\sum_{1\leq i_1,i_2,i_3\leq n}\tilde{h}(X^{(1)}_{i_1},X^{(2)}_{i_2},X^{(3)}_{i_3})+n\sum_{1\leq i_2,i_3\leq n}h'(X^{(2)}_{i_2},X^{(3)}_{i_3})\Big)
\end{equation}

Given the orthogonality of $\phi_i$ and $\phi_j$ for $i\neq j$ it is easy to check that $h'(y,z)$ is also a canonical kernel (of degree two). We control the degree two term first. For this term we use the exponential tail bound \cite[Cor.3.4]{GinLatZin}, which gives that with probability at least $1-\alpha$, we have \[\sum_{1\leq i_2,i_3\leq n}h'(X^{(2)}_{i_2},X^{(3)}_{i_3})\lesssim_\alpha \max\{A,B,C,D\}\]
where 
\begin{align*}
A&=\|h'\|_{\infty}\\
B&=\sqrt{n(\|\E_X[h'^2(X,\cdot)]\|_{\infty}+\|\E_Y[h'^2(\cdot,Y)]\|_{\infty})}\\
C&=n\sqrt{\E_{XY}(h'^2(X,Y))}\\
D&=n\|h'\|_{L^2\rightarrow L^2}
\end{align*}
Here \[\|h'\|_{L^2\rightarrow L^2}:=\sup{\{\E[h'(X,Y)\zeta_1(X)\zeta_2(Y)]: \E[\zeta^2_1(X)]\leq 1,\ \E[\zeta^2_2(X)]\leq 1\}}\]
For $\zeta_1,\zeta_2\in L^2(\Omega)$ such that $\E[\zeta^2_1(X)]\leq 1, \E[\zeta^2_2(X)]\leq 1$ we have

\begin{align*}
\|h'\|_{L^2\rightarrow L^2}&=\sup{\E[\sum_{k>R}\lambda^2_k\phi_k(X)\phi_l(X)\zeta_1(X)\phi_k(Y)\phi_l(Y)\zeta_2(Y)]}\\
&\leq \sup{\sum_{k>R}\lambda^2_k\E[\phi_k\phi_l\zeta_1]\E[\phi_k\phi_l\zeta_2]}\\
&\leq \|\phi^2_l\|_{\infty}\sum_{k>R}\lambda^2_k\E[|\phi_k\zeta_1|]E[|\phi_k\zeta_2|]\\
&\leq \|\phi^2_l\|_{\infty}\sum_{k>R}\lambda^2_k\\
&=\|\phi^2_l\|_{\infty}b_{2,R}
\end{align*}
where we used Cauchy-Schwarz in the third step and the fact that $\E[\zeta^2_1]\leq 1, \E[\zeta^2_2]\leq 1$. This implies that $D\leq n\|\phi^2_l\|_{\infty}b_{2,R}$. We claim that this bound also holds for the term $C$. Indeed, we have 
\begin{align*}
\E_{XY}[h'^2(X,Y)]&=\sum_{k,k'>R}\lambda^2_k\lambda^2_{k'}\E[\phi_k\phi^2_l\phi_{k'}]^2\\
&\leq \|\phi^2_l\|^2_{\infty}\sum_{k,k'>R}\lambda^2_k\lambda^2_{k'}\E[|\phi_k\phi_{k'}|]^2\\
&\leq \|\phi^2_l\|^2_{\infty}\sum_{k,k'>R}\lambda^2_k\lambda^2_{k'}\\
&= \|\phi^2_l\|^2_{\infty}b^2_{2,R}
\end{align*}
where we used Cauchy-Schwarz in the third step. Which proves that $C\leq n\|\phi^2_l\|_{\infty}b_{2,R}$. It is easy to check that it holds 
\begin{align*}
A&\leq \|\phi^2_l\|_{\infty} b_R\V_2(R)\\
B&\leq \sqrt{n}\|\phi^2_l\|_{\infty}\sqrt{b_{3,R}\V_2(R)}
\end{align*}

This implies that with probability larger than $1-\alpha$ we have \begin{equation}\label{eq:conclu1}\frac1{n^2}\sum_{1\leq i_2,i_3\leq n}h'(X^{(2)}_{i_2},X^{(3)}_{i_3})\lesssim_\alpha \frac1{n}\|\phi^2_l\|_{\infty}b_{2,R}\end{equation}

For the term $\frac{1}{n^3}\sum_{1\leq i_1,i_2,i_3\leq n}\tilde{h}(X^{(1)}_{i_1},X^{(2)}_{i_2},X^{(3)}_{i_3})$ we use \cite[Cor.3]{Adam} which generalizes the tail inequality \cite[Cor.3.4]{GinLatZin} for higher order $U$-statistics. In the notation of \cite{Adam} we have $I_3=\{1,2,3\}$ and $\mathcal{J}=\{J_1,\cdots,J_k\}$ is a partition of a set $I\subset I_3$ ($J_1,\cdots,J_k$ are non empty, hence $k$ is at most $3$). By \cite[Cor.3]{Adam} we have \begin{equation}\label{eq:probaU3}\P\Big(\sum_{1\leq i_1,i_2,i_3\leq n}\tilde{h}(X^{(1)}_{i_1},X^{(2)}_{i_2},X^{(3)}_{i_3})>t\Big)\leq K\exp{\Big\{-\min_{I\subset I_3,\mathcal{J}\in\mathcal{P}_I}\Big(\frac{t}{n^{|I|/2}\|\|\tilde{h}\|_{\J}\|_{\infty}}\Big)^{\frac{2}{deg(\J)+2|I^c|}}\Big\}}\end{equation}
where $deg(\J)$ is the degree of the partition, that is, the number of its elements and for a set $I\subset I_3$, $|I|$ is its cardinality. In this case,  \[\|\tilde{h}\|_{\J}=\sup\{\E_I[\tilde{h}(X_I) \prod^{deg(\J)}_{i=1}\zeta_{J_i}(X_{J_i})]: \E[\zeta^2_{J_i}(X_{J_i})]\leq 1, \text{ for }i=1,\cdots,deg(\J) \}\]
Here $\E_I$ represents the expectation with respect to the elements in $I$. For example, if $I={1,2}$ we have $\E_I[\tilde{h}(X_I)]=\E_{X,Y}[\tilde{h}(X,Y,\cdot)]$. From \eqref{eq:probaU3} we see that the higher order term is of the form $n^{3/2}\|\|\tilde{h}\|_{\J}\|_{\infty}$ where $\J$ is a partition of $I_3$. In this case we need to bound the the following terms \[\sup{\{\E_{X,Y,Z}[\tilde{h}(X,Y,Z)\zeta_{1,2,3}(X,Y,Z)]:\E[\zeta^2_{1,2,3}]\leq 1\}},\]
 \[\sup{\{\E_{X,Y,Z}[\tilde{h}(X,Y,Z)\zeta_{1,2}(X,Y)\zeta_3(Z)]:\E[\zeta^2_{1,2}]\leq 1,E[\zeta^2_{3}]\leq 1\}},\]
 \[\sup{\{\E_{X,Y,Z}[\tilde{h}(X,Y,Z)\zeta_{1,3}(X,Z)\zeta_2(Y)]:\E[\zeta^2_{1,3}]\leq 1,E[\zeta^2_{2}]\leq 1\}}, \]
  \[\sup{\{\E_{X,Y,Z}[\tilde{h}(X,Y,Z)\zeta_{2,3}(Y,Z)\zeta_1(X)]:\E[\zeta^2_{2,3}]\leq 1,E[\zeta^2_{1}]\leq 1\}}, \]
   \[\sup{\{\E_{X,Y,Z}[\tilde{h}(X,Y,Z)\zeta_{1}(X)\zeta_2(Y)\zeta_3(Z)]:\E[\zeta^2_{1}]\leq 1,E[\zeta^2_{2}]\leq 1,\E[\zeta^2_{3}]\leq 1\}} \]
All those terms can be bounded by the same quantity, which comes from the definition of $\tilde{h}$ as a product. We have, for instance \begin{align*}
\E[\tilde{h}(X,Y,Z)\zeta_{1}(X)\zeta_2(Y)\zeta_3(Z)]&=\sum_{k,k'>R}\lambda_k\lambda_{k'}\E[\phi_{k}\phi_{k'}\zeta_1-\delta_{kk'}\zeta_1]\E[\phi_k\phi_l\zeta_2]\E[\phi_{k'}\phi_l\zeta_3]\\
&\leq \|\phi^2_l\|_{\infty}\sum_{k,k'>R}\lambda_k\lambda_{k'}\E[(\phi_{k}\phi_{k'}-\delta_{kk'})\zeta_1]\\
&\leq \|\phi^2_l\|_{\infty}\sum_{k,k'>R}\lambda_k\lambda_{k'}\|\phi_k\|_{\infty}\E[(\phi_{k'}-\frac{\delta_{kk'}}{\|phi_k\|_{\infty}})\zeta_1)]\\
&\leq 2 \|\phi^2_l\|_{\infty}b_R\V_3(R)
\end{align*}
where from the third to the fourth line we used Cauchy-Schwarz and $\V_3(R):=\sum_{k>R}|\lambda_k|\|\phi_k\|_{\infty}$. The rest of the terms are bounded similarly to obtain, using \eqref{eq:probaU3} \begin{equation}\label{eq:conclu2}\frac1{n^3}\sum_{1\leq i_1,i_2,i_3\leq n}\tilde{h}(X^{(1)}_{i_1},X^{(2)}_{i_2},X^{(3)}_{i_3})\lesssim_\alpha \frac1{n^{3/2}} \|\phi^2_l\|_{\infty}b_R\V_3(R)\end{equation}

From \eqref{eq:conclu1} and \eqref{eq:conclu2} it is direct that with probability larger than $1-\alpha$ \[\sum^R_{l=1}\|E_R\phi_l\|^2\lesssim_\alpha \frac{b_{2,R}}{n}\sum^R_{l=1}\|\phi^2_l\|_{\infty}=\frac{1}{n}b_{2,R}\V'_1(R)\]


\subsubsection{Tail bound for $\|{\Phi^\perp_R}^TE_R\Phi^\perp_R\|_{op}$}\label{sec:boundER}
Since we do not know $\Phi^\perp_R$ explicitly, we will use the bound $\|{\Phi^\perp_R}^TE_R\Phi^\perp_R\|_{op}\leq \|E_R\|_{op}$. Given that the columns of $\Phi_R$ are asymptotically orthogonal to those of $E_R$, those of $\Phi^\perp_R$ will be asymptotically aligned. That justify the use of the aforementioned bound. 

We recall that $n(E_R)_{ij}=\sum_{k>R}\lambda_k\phi_k(X_i)\phi_k(X_j)$. Let $I^+$ (resp. $I^-$) be set of integers $k$, with $k$ larger than $R$, such that $\lambda_k\geq0$(resp. $\lambda_k<0$). We decompose $E_R$ in $E^+_R$ and $E^-_R$ as follows \[n(E_R)_{ij}=\underbrace{\sum_{k\in I^+}\lambda_k\phi_k(X_i)\phi_k(X_j)}_{E^+_R}-\underbrace{\sum_{k\in I^-}|\lambda_k|\phi_k(X_i)\phi_k(X_j)}_{E^-_R}\]

Given that the matrices $E^+_R$ and $E^-_R$ are kernel matrices with positive semidefinite(p.s.d) kernels, both are positive semidefinite matrices. From Weyl perturbation theorem we have \[\|E_R\|_{op}\leq \|\frac1nE^+_R\|_{op}+\|\frac1nE^-_R\|_{op}\]

Given that $E^+_R$ and $E^-_R$ are p.s.d we can bound the operator norm by their trace. For instance, $\|E^+_R\|_{op}\leq \operatorname{Tr(E^+_R)}$, where $\operatorname{Tr}$ is the trace of the matrix. Observe that $\operatorname{Tr}(E^+_R)$ and $\operatorname{Tr}(E^-_R)$ are sums of independent random variables.
 
From the strong law of large numbers we have \begin{align*}
\operatorname{Tr}(E^+_R)&\rightarrow \int_\Omega \sum_{k\in I^+}\lambda_k\phi^2_k(x)d\mu(x)=\sum_{k\in I^+}\lambda_k=:b^+_{R}\\
\operatorname{Tr}(E^-_R)&\rightarrow \int_\Omega \sum_{k\in I^-}\lambda_k\phi^2_k(x)d\mu(x)=\sum_{k\in I^-}|\lambda_k|=:b^-_R\\
\end{align*}
both convergences hold in the a.s sense. Notice that \[\|E^+_R(X_i,X_i)\|_{\infty}\leq\|\sum_{k\in I^+}\lambda_k\phi^2_k\|_{\infty} =:\V^+_2(R)\] for all $1\leq i\leq n$. On the other hand, it is easy to see that \begin{align*}Var_X(\sum_{k\in I^+}\lambda_k\phi^2_k(X))&\leq \E_X\Big((\sum_{k\in I^+}\lambda_k\phi^2_k(X))^2\Big)\\ &\leq\|\sum_{k\in I^+} \lambda_k\phi^2_k\|_{\infty}\E[\sum_{k'}\lambda_{k'}\phi^2_k]\\&\leq \V^+_2(R)b_R
\end{align*}
with an analogous result for the indices $I^-$.  We use Bernstein inequality to obtain with probability larger than $1-\alpha$\[\|E^+_R\|_{op}\lesssim b^+_R+\sqrt{\frac{\V^+_2(R)b_R\log{1/\alpha}}{n}}\vee \frac{\V^+_2(R)\log{1/\alpha}}{n}\]

and 
\[\|E^-_R\|_{op}\lesssim b^-_R+\sqrt{\frac{\V^-_2(R)b_R\log{1/\alpha}}{n}}\vee \frac{\V^-_2(R)\log{1/\alpha}}{n}\]
Then we obtain that with probability larger than $1-\alpha$ we have \[\|E_R\|_{op}\lesssim b_R+\sqrt{\frac{\V_2(R)b_R\log{1/\alpha}}{n}}\vee \frac{\V_2(R)\log{1/\alpha}}{n}\]

which proves \eqref{eq:prop_2}.

\subsection{Proof of Lemma \ref{lem:tau}}
The proof is based on the following three lemmas. 

\begin{lemma}\label{lem:projs}
Let $P_1$ and $P_2$ be two $n\times n$ projection matrices onto the linear subspaces of $R^n$ $U_1$ and $U_2$ respectively. If $U_1$ and $U_2$ are orthogonal, then for any $n\times n$ symmetric matrix $A$, it holds 
\begin{align*}
\|P_2AP_1+P_1AP_2\|_{op}&\leq 2 \min\{\max_{x\in U_1,\|x\|\leq1}\|Ax\|,\max_{x\in U_2,\|x\|\leq1}\|Ax\|\}\\
\|P_1AP_1\|_{op}&\leq \max_{x\in U_1,\|x\|\leq1}\|Ax\|
\end{align*}
\end{lemma}

\begin{proof}
By Courant-Fisher characterization we have 
\begin{align*}
\|P_2AP_1+P_1AP_2\|_{op}&=\max_{\|x\|=1} x^TP_2AP_1x+x^TP_1AP_2x\\
&\leq\max_{\substack{\|x^{(1)}\|^2+\|x^{(2)}\|^2=1\\ x^{(1)}\in U_1\\ x^{(2)}\in U_2}}{x^{(2)}}^TAx^{(1)}+{x^{(1)}}^TAx^{(2)}\\
&\leq2\max_{\substack{\|x^{(1)}\|^2+\|x^{(2)}\|^2=1\\ x^{(1)}\in U_1\\ x^{(2)}\in U_2}}\|x^{(2)}\| \|Ax^{(1)}\|\\
&\leq2\max_{\substack{\|x^{(1)}\|\leq 1\\ x^{(1)}\in U_1}}\sqrt{1-\|x^{(1)}\|^2}\|Ax^{(1)}\|\\
&\leq2\max_{\substack{\|x^{(1)}\|\leq 1\\ x^{(1)}\in U_1}}\|Ax^{(1)}\|\\
\end{align*}
The proof of $\|P_2AP_1+P_1AP_2\|_{op}\leq2\max_{\substack{\|x^{(2)}\|\leq 1\\ x^{(2)}\in U_2}}\|Ax^{(2)}\| $ and $\|P_1AP_1\|_{op}\leq \max_{x\in U_1,\|x\|\leq1}\|Ax\|$ are analogous.
\end{proof}

Recall that for a given $R$ the event $\mathcal{E}_\tau$ holds with probability larger than $1-Re^{-\frac{t^2n}{\V_1(R)}}$. 

\begin{lemma}\label{lem:control_alpha}
Under the event $\mathcal{E}_\tau$, for $0<\tau<1$, we have for all $x\in \spa\{\phi_1,\cdots,\phi_R\}\subset \R^n$ such that $x=\sum^R_{i=1}\alpha_i\phi_i$, for $\alpha_i\in\R$, and $\|x\|=1$, we have \[\sum^R_{i=1}\alpha^2_i\leq \frac{1}{1-\tau}\]
\end{lemma}
\begin{proof}
Notice that $x$ can be written as $x=\Phi_R\alpha$, where $\alpha:=(\alpha_1,\cdots,\alpha_R)^T$. Given that $x^Tx=1$, we have that \begin{align*}
|\alpha^T\alpha-1|&=|\alpha^T\alpha-\alpha^T\Phi^T_R\Phi_R\alpha|\\
&=|\alpha^T(\I_R-\Phi^T_R\Phi_R)\alpha|\\
&\leq \alpha^T\alpha\frac{|\alpha^T(\I_R-\Phi^T_R\Phi_R)\alpha|}{\alpha^T\alpha}\\
&\leq \alpha^T\alpha\|\I_R-\Phi^T_R\Phi_R\|_{op}
\end{align*}
Under the event $\mathcal{E}_\tau$, we have \begin{equation}\label{eq:alphas}\frac{|\alpha^T\alpha-1|}{\alpha^T\alpha}\leq \frac{1}{1-\tau}\end{equation}
Notice that when if $\alpha^T\alpha<1$, there is nothing to prove as $\frac{1}{1-\tau}>1$. If $\alpha^T\alpha>1$, we deduce from \eqref{eq:alphas} that $\alpha^T\alpha\leq\frac{1}{1-\tau}$.
\end{proof}

\begin{lemma}\label{lem:boundAx}
Under the event $\mathcal{E}_\tau$, for all $x\in \spa\{\phi_1,\cdots,\phi_R\}\subset \R^n$ such that $\|x\|=1$ and for any $n\times n$ real matrix $A$ it holds \[\|Ax\|\leq \sqrt{\frac{1}{1-\tau}}\sqrt{\sum^R_{i=1}\|A\phi_i\|^2}\]
\end{lemma}
\begin{proof}
Let $\alpha_i\in\R$ for $1\leq i\leq R$ be such that $x=\sum^R_{i=1}\alpha_i\phi_i$.
 \begin{align*}
\|Ax\|&=\|\sum^R_{i=1}\alpha_iA\phi_i\|\\
&\leq \sqrt{\sum^R_{i=1}\alpha^2_i}\sqrt{\sum^R_{i=1}\|A\phi_i\|^2}\\
&\leq \sqrt{\frac{1}{1-\tau}}\sqrt{\sum^R_{i=1}\|A\phi_i\|^2}
\end{align*}
where in the first step we used triangle inequality and Cauchy-Schwarz and in the last step we used Lemma \ref{lem:control_alpha}, under the assumption that $\mathcal{E}_\tau$ holds. 
\end{proof}

\begin{proof}[Proof of Lemma \ref{lem:tau}]
Recall that $A:=P_1E_RP_2+P_2E_RP_1+P_1E_RP_1$. Applying Lemma \label{lem:projs} we obtain that \[\|A\|_{op}\lesssim\max_{\phi\in\operatorname{Sp}(\Phi_R),\|\phi\|=1}\|E\phi\| \]
because $P_1$ is the projection onto $\operatorname{Sp}(\Phi_R)$. This proves the first inequality. For the second inequality, note that in the event $\mathcal{E}_\tau$ we obtain, using Lemma \ref{lem:boundAx} \[\max_{\phi\in\operatorname{Sp}(\Phi_R),\|\phi\|=1}\|E\phi\|\leq \sqrt{\frac{1}{1-\tau}}\sqrt{\sum^R_{l=1}\|E_R\phi_l\|^2}\]
Under the event $\mathcal{E}_{\alpha}$ (the event such that \eqref{eq:prop_1} holds) we have
\[\max_{\phi\in\operatorname{Sp}(\Phi_R),\|\phi\|=1}\|E\phi\|\lesssim_\alpha \sqrt{\frac{1}{1-\tau}}\gamma_1(n,R)\]
By Prop. \ref{prop:tail_bounds}, we have that $\P(\mathcal{E}_\alpha)\geq 1-\alpha$. By definition of $\tau_{n,R,\alpha}$, we have $\P(\mathcal{E}_{\tau_{n,R,\alpha}})\geq 1-\alpha$. This implies that $\P(\mathcal{E}_\alpha\cap \mathcal{E}_{\tau_{n,R,\alpha}})\geq 1-2\alpha$.
\end{proof}

\subsection{Proof of Proposition \ref{prop:Rsmaller}}

From  \eqref{eq:pert_fin1}, Prop. \ref{prop:con_Gram} and \eqref{eq:boundA} we deduce that with probability larger than $1-\alpha$, we have \[|\lambda_i(T_n)-\lambda_i(M)|\lesssim_{\alpha,\tau}|\lambda_i(M)|\sqrt{\frac{\V_1(R)\log{R/\alpha}}{n}}+\gamma_1(n,R)\]
Given that $|\lambda_i(M)|\leq (|\lambda_i|\vee \gamma_2(n,R))$ , because of the block structure of $M$, the statement follows.
\subsection{Proof of Proposition \ref{prop:Rlarger}}
From $T_n=\Phi_R\Lambda_R\Phi^T_R+E_R$ we see, using Weyl's inequality, that \[|\lambda_i(T_n)-\lambda_i(\Phi_R\Lambda_R\Phi^T_R)|<\|E_R\|_{op},\] which implies that with probability at least $1-\alpha$ we have \[|\lambda_i(T_n)|\lesssim_\alpha\gamma_2(n,R),\] because $\lambda_i(\Phi_R\Lambda_R\Phi^T_R)=0$ for $i>R$. On the other hand, by definition of $\gamma_2(n,R)$(because it contains the tail $b_R$ as summand) we have $|\lambda_i|\leq \gamma_2(n,R)$ for $i\geq R$. Then we have \[|\lambda_i(T_n)-\lambda_i|\lesssim_\alpha \gamma_2(n,R)\] with the required probability.

\subsection{Proof of Thm. \ref{thm:theo1}}\label{app:proof_thm1}

Fix $i\in\N$. Notice that the set in the definition of $R(i)$ is non-empty given that $b_R\to 0$ and $Rb_{2,R}\to 0$, as $R\to\infty$. Indeed, we have that the operator $T_W$ is trace class, which implies that $\sum_{k}|\lambda_k|<\infty$ (so $b_R\to 0$ as $R$ grows). From the latter, and the fact the eigenvalues are ordered in decreasing order, we deduce that $\lambda_k=\O(k^{-1-\varepsilon})$, for some $\varepsilon>0$. Given that $Rb_{2,R}\leq R\lambda_Rb_R$ and that $R\lambda_R=\O(R^{-\varepsilon})$, we have that $Rb_{2,R}\to 0$ as $r\to\infty$ (because $b_R\to 0$). Define $\tilde{b}_R=\sum_{k>R}|\lambda_k|\vee \sqrt{R\sum_{k>R}\lambda^2_k}$. Note that for $R(i)$ we have by definition that $|\lambda_i|>\tilde{b}_{R(i)}$ and define  \[n'_0:=\ceil*{\frac{\V_2(R(i))}{(\lambda_i-\tilde{b}_{R(i)})^2}}\]
Then for all $n\geq n'_0$ we have \[|\lambda_i|>\tilde{b}_{R(i)}+\sqrt{\frac{\V_2(R(i))}{n}}=\gamma_2(n,R)\]
Define $n''_0:=\min\{n\in\N: \tau_{n,R(i),\alpha}<\frac12\}$. For $n\geq n_0:=\min\{n'_0,n''_0\}$, we have, using Prop.\ref{prop:Rsmaller},  with probability larger than $1-\alpha$\[|\lambda_i(T_n)-\lambda_i|\lesssim |\lambda_i|\sqrt{\frac{\V_1(R(i))\log{R(i)/\alpha}}{n}}+\gamma_1(n,R(i))\]
Given that $\gamma_1(n,R)=\sqrt{b_{2,R}\V'_1(R)}{n}$ we have that \[|\lambda_i|\sqrt{\frac{\V_1(R(i))}{n}}\geq \gamma_1(n,R(i))\]Indeed, we have that $\V'_1(R)\leq R\V_1(R)$ (indeed $\V_1(R)\leq R\max_{1\leq k\leq R}\|\phi^2_k\|_\infty\leq R\V_1(R)$) and by definition of $R(i)$ we have that $|\lambda_i|>\sqrt{Rb_{2,R}}$. It follows that with probability larger than $1-\alpha$ we have 
\[|\lambda_i(T_n)-\lambda_i|\lesssim |\lambda_i|\sqrt{\frac{\V_1(R(i))\log{R(i)/\alpha}}{n}}\]

\begin{remark}
Under the hypothesis $\H_1$, we have that $|\lambda_i|=\O(i^{-\delta})$, $b_R=\O(R^{1-\delta})$ and $b_{2,R}=\O(R^{1-2\delta})$ (which implies that $\sqrt{Rb_{2,R}}=O(R^{1-\delta})$). It is easy to see that for $R=\O(i^{\frac{\delta}{\delta-1}})$ we have $|\lambda_i|>b_R\vee \sqrt{Rb_{2,R}}$ which implies that $R(i)=\O(i^{\frac{\delta}{\delta-1}})$ (observe that $\delta>1$). A similar analysis leads to $R(i)=\O(i)$ in the case of exponential decay hypothesis.
\end{remark}

\subsection{Proof of Theorem \ref{thm:theo2}}

The idea is for each index $i$ to use either Prop. \ref{prop:Rsmaller} or Prop. \ref{prop:Rlarger} (the one delivering the tighter bound). We can see the results in this section as finding a rule that tell us how to select the truncation parameter $R$ best adapted for each $i$. 

To determine a rate from Prop. \ref{prop:Rsmaller}, it is important to know whether or not we have $\lambda_i>\gamma_2(n,R)$.

We will further precise when this holds under each regularity hypothesis $\H_1$, $\H_2$ or $\H_3$. In each case we will have a series of lemmas ending with a rate for $\operatorname{Err}_i=|\lambda_i(T_n)-\lambda_i|$. From here, to the rest of this section, all the inequalities must be understood as holding with probability as least $1-\alpha$. 


\subsubsection{Hypothesis $\H_1$.}\label{proof:H1}

Here we assume $|\lambda_i|=i^{-\delta}$ and $\|\phi_{i}\|_{\infty}=i^s$, with $\delta>2s+1$. 
We will introduce $\beta$ and $\delta'$ such that $i=n^\beta$ and $R=i^{\delta^\prime}$. Observe that this imply that in our parametrization we have $\beta\delta^\prime\leq 1$, because $R<n$. 
We start by characterize the order of $\gamma_2(n,R)$, recalling from that $\gamma_2(n,R)=b_R+\frac{1}{\sqrt{n}}(\V_2(R))^{\frac12}$

\begin{lemma}
Under $\H_1$ we have \[\V_2(R)^\frac12\lesssim R^{s+1-\delta}\] and consequently \[\gamma_2(n,R)\lesssim R^{1-\delta}+R^{s+1-\delta}n^{-\frac12}\]
\end{lemma}
\begin{proof}
The first inequality follows by plugin the regularity conditions in the definition of $\V_2$ and $\V_3$ and using estimate $\sum_{k>R}k^{-p}=\O(R^{1-p})$ when $p>1$. Indeed, we have 
\begin{align*}
\V_2(R)&=\|\sum_{k>R}\lambda_k \phi_k\otimes \phi_k\|_{\infty}b_R=\O(R^{2(s+1-\delta)})\\
\end{align*}
which implies that $\V_2(R)^\frac12=\O(R^{s+1-\delta})$. The order of $\gamma_2(n,R)$ follows directly by inserting this and using the definition of $b_R$. 
\end{proof}
Note that the conclusion of the previous lemma can be restated, considering the notation introduced in this section, as \begin{equation}\label{eq:ord_gamma2}
\gamma_2(n,R)=\O(n^{\beta\delta^\prime(1-\delta)})+\O(n^{\beta\delta^\prime(s+1-\delta)-\frac12})\end{equation}
The following lemma, describe the condition $\lambda_i>\gamma_2(n,R)$.
\begin{lemma}\label{lem:rel1}
Assume $s\geq 1$. If $\beta\delta^\prime\geq \frac{1}{2s}$ , then we have 
\[\begin{cases}
\lambda_i>\gamma_2(n,R)\quad \text{ if }\beta\delta^\prime\geq \frac{\beta\delta-1/2}{\delta-s-1}\\
\lambda_i\leq\gamma_2(n,R)\quad\text{ otherwise}
\end{cases}\]
On the other hand, if $\beta\delta^\prime<\frac{1}{2s}$ then \[\begin{cases}
\lambda_i>\gamma_2(n,R)\quad \text{ if }\beta\delta^\prime\geq \frac{\beta\delta}{\delta-1}\\
\lambda_i\leq\gamma_2(n,R)\quad\text{ otherwise}
\end{cases}\]
\end{lemma}
\begin{proof}
The idea is to express $\lambda_i$ and $\gamma_2(n,R)$ in terms of $n$ and compared them. We have \begin{align*}
\lambda_i&=\O(n^{-\beta\delta})\\
\gamma_2(n,R)&=\O(n^{\beta\delta^\prime(1-\delta)})+\O(n^{\beta\delta^\prime(s+1-\delta)-\frac12})
\end{align*}
It is clear that \[\gamma_2(n,R)=\begin{cases} \O(n^{\beta\delta^\prime(1-\delta)})\quad \text{ if } \beta\delta<\frac1{2s}\\   \O(n^{\beta\delta^\prime(s+1-\delta)-\frac12})\quad \text{ otherwise }     \end{cases}\]

From this is clear that if $\beta\delta^\prime\geq \frac{\beta\delta}{\delta-1}$, we will have that $\lambda_i>\gamma_2$, in the case $\beta\delta^\prime<\frac1{2s}$. In the case $\beta\delta^\prime\geq\frac1{2s}$ we need to verify $-\beta\delta>\beta\delta^\prime(s+1-\delta)-\frac12$, which is true if $\beta\delta^\prime>\frac{\beta\delta-\frac12}{\delta-s-1}$.
\end{proof}
By definition of $\gamma_1(n,R)$, we have the following \begin{equation}\label{eq:ordergamma_1}
\gamma_1(n,R)=\O(R^{s+1-\delta}n^{-1/2})=\O(n^{\beta\delta^\prime(s+1-\delta)-\frac12})
\end{equation}

The following lemma studies the condition $\tau_{R,n,\alpha}<1$
\begin{lemma}\label{lem:tau_H1}
If $\beta\delta^\prime<\frac{1}{2s+1}$ we have \[\tau_{R,n,\alpha}<1\]
\end{lemma}
\begin{proof}
This is an easy consequence of the fact that $\tau_{R,n,\alpha}=\O(R^{s+1/2}n^{-1/2})=\O(n^{\beta\delta^\prime(s+\frac12)-\frac12})$, then when $\tau_{R,n,\alpha}=o(1)$, which proves the assertion. Note that here the constants are not important as we can always divide our matrices an operators by a particular constant and the analysis remains unchanged.
\end{proof}

\begin{proof}[Proof of Thm.\ref{thm:theo2} under $\H_1$]
Since $s\geq1$ we have that $\delta>2s+1=3$. Assume $\beta<\frac1{2s+1}\frac{\delta-1}{\delta}$. Take $\delta^\prime=\frac{\delta}{\delta-1}$, in this case  $\beta\delta^\prime\leq \frac1{2s+1}<\frac1{2s}$, thus the assumptions for Lemma \ref{lem:tau_H1} are verified. 
Then by Lemma \ref{lem:rel1}  we have $\lambda_i>\gamma_2(n,R)$.
By Proposition \ref{prop:Rsmaller} \[\Err_i\lesssim_\alpha |\lambda_i|\sqrt{\frac{\V_1(R)\log{R}}{n}}+\gamma_1(n,R)\]

with probability larger than $1-\alpha$, which in this case, given \eqref{eq:ordergamma_1}, is equivalent to \[\Err_i\lesssim_\alpha i^{\delta+\frac{\delta}{\delta-1}(s+\frac12)}n^{-\frac12}=\O(n^{-\beta\delta+\beta\frac{\delta}{\delta-1}(s+\frac12)-\frac12})\]
When $\beta>\frac{1}{2s}$ we use Prop. \ref{prop:Rlarger} obtaining \[\Err_i\lesssim_\alpha \gamma_2(n,R)\]
The order of $\gamma_2(n,R)$ is given by \eqref{eq:ord_gamma2}. Note that here we must take $R$ smaller or equal than $i$. When $R=i$ we have $\delta^\prime=1$ and in this case the order is \[\Err_i=\O(n^{\beta\delta^\prime(s+1-\delta)-\frac12})=\O(n^{-\beta\delta+\beta(s+1)-\frac{1}{2}})\] 
In the case $\frac{1}{2s+1}\frac{\delta-1}{\delta}\leq\beta\leq \frac{1}{2s}$ we use Prop.\ref{prop:Rlarger} with $\delta'=1$, obtaining \[Err_i\lesssim_\alpha i^{1-\delta}=i^{-\delta+1+\frac12\frac{\log i}{\log n}}n^{-\frac12}\] where we used Lemma \ref{lem:rel1}. In this case, it is easy to verify that $(s+\frac12)\frac{\delta}{\delta-1}\geq \frac1{2\beta}\geq s$. Which implies that $\frac12\frac{\log i}{\log n}\leq (s+\frac12)\frac{\delta}{\delta-1}$.  

The result follows by noticing that by definition $\beta=\log{i}/\log{n}$. 
\end{proof}


As we already remarked, Theorem \ref{thm:theo2} offers a bound on $\Err_i$ which is valid for varying $i$, but in light of the results of this section, the bound presented in Thm. \ref{thm:theo2} is not the tighter we can obtain with the same method. We opt to not include the tighter results in the main paper for better clarity and better readability, given that the improvement is marginal. Tighter results can be obtained by direct use of Lemmas \ref{lem:rel1}-\ref{lem:tau_H1}, which give a better resolution for the values of $\frac1{2s+1}\frac{\delta-1}{\delta}<\beta<\frac{1}{2s+1}$ than those in Theorem \ref{thm:theo2}. 
 Such improvements follow a similar argument that in the proof of Theorem \ref{thm:theo2}. 
 

 The proof presented here do not cover the case $s=0$, for hypothesis $\text{H}_1$, and we will prove it separately. The reason is that in the case $s=0$, there are lower regularity kernels that are admissible (satisfying $\text{H}_1$) which makes the term $R^{1-\delta}$ converging to zero very slowly. Take for instance a kernel with regularity $\delta=1+\epsilon$ with $\epsilon$ close to $0$. To cover those cases we need a small refinement of our bounds, which will be achieved by using the same bounds iteratively. In particular, we will refine the bound for the operator norm of the residual matrix $E_R$. For that, define $\V_1(R,R')=\|\sum^{R'}_{k=R}\phi^2_k\|_{\infty}$ for $R',R\in\N$ with $R'>R$ and let $\Phi_{R,R'}$ be the matrix with columns $1/\sqrt n(\phi_k(X_1),\phi_k(X_2),\cdots,\phi_k(X_n))^T$, for $R\leq k\leq R'$ and $\Lambda_{R,R'}$ the diagonal matrix with diagonal equals to $\lambda_R,\cdots,\lambda_{R'}$
 \begin{lemma}
Let $W$ be a kernel satisfying $\text{H}_1$ with $s=0$, then we have \[|\lambda_i(T_n)-\lambda_i|\lesssim_\alpha i^{\frac12-\delta}n^{-\frac12}\]
for all $1\leq i\leq n$. 
\end{lemma}
\begin{proof}
Take $1\leq i\leq n$. We will consider the sequence of values $R_j=(j+1)i$ for $j\in \{0,1,\cdots, k\}$, where $k\in \N$ will be determined later. By Ostrowskii's inequality and Weyl's inequality we have \[|\lambda_i(T_n)-\lambda_i|\lesssim_\alpha \lambda_i\sqrt{\frac{\V_1(0,R_0)}{n}}+\|E_{R_0+1}\|_{op}\]
On the other hand, we note that for any $R,R'\in\N$ such that $R'>R$, the following decomposition holds \[E_{R}=\Phi_{R,R'}\Lambda_{R,R'}\Phi^T_{R,R'}+E_{R'+1}\]
which is analogous to \eqref{eq:pert1}. Using Ostrowskii and Weyl's inequalities we obtain for $R=R_0+1$ and $R'=R_1$
\[\|E_{R_0+1}\|_{op}\lesssim_\alpha \lambda_{R_0+1}\sqrt{\frac{\V_1(R_0+1,R_1)}{n}}+\|E_{R_1+1}\|_{op}\]
More generally, we have the recurrence \[\|E_{R_j+1}\|_{op}\lesssim_\alpha \lambda_{R_j+1}\sqrt{\frac{\V_1(R_j+1,R_{j+1})}{n}}+\|E_{R_{j+1}+1}\|_{op}\]
for $j\in \{0,1\cdots,k\}$. 
This implies that 
\begin{align*}
|\lambda_i(T_n)-\lambda_i| &\lesssim_\alpha \lambda_i\sqrt{\frac{\V_1(0,R_0)}{n}}+\sum^{k-1}_{j=0}\lambda_{R_j+1}\sqrt{\frac{\V_1(R_j+1,R_{j+1})}{n}}+\|E_{R_k+1}\|_{op}\\
 &\lesssim_\alpha  \lambda_i\sqrt{\frac{\V_1(0,R_0)}{n}}+\sum^{k-1}_{j=0}\lambda_{R_j+1}\sqrt{\frac{\V_1(R_j+1,R_{j+1})}{n}}+b_{R_k+1}+\sqrt{\frac{\V_2(R_k+1)}{n}}
\end{align*}
where in the last step we used the bound for $\|E_R\|_{op}$ proved in Sec.\ref{sec:boundER}. Given that we assume $s=0$, we have that $\V_1(0,R_0)=\O(i)$ and $\V_1(R_j+1,R_{j+1})=\O(i)$, by the definition of $R_j$. On the other hand we have $\sqrt{\V_2(R_k+1)}=\O\big((R_k+1)^{1-\delta}\big)$. Gathering all this, we obtain the following 
\begin{align*}
|\lambda_i(T_n)-\lambda_i| &\lesssim_\alpha i^{-\delta+\frac12}n^{-\frac12}+i^{-\delta+\frac12}n^{-\frac12}\sum^{k-1}_{j=0}(j+1)^{-\delta}+(R_k+1)^{1-\delta}+(R_k+1)^{1-\delta}n^{-\frac12}\\
 &\lesssim_\alpha i^{-\delta+\frac12}n^{-\frac12}+i^{-\delta+\frac12}n^{-\frac12}\zeta(\delta)+(R_k+1)^{1-\delta}+(R_k+1)^{1-\delta}n^{-\frac12}
\end{align*}
where $\zeta(s)$ is the Riemann Zeta function, which is finite given that $\delta>1$. Given that $(R_{k}+1)^{1-\delta}=\O(i^{1-\delta}(k+1)^{1-\delta})$ we choose $k$ such that $(k+1)^{1-\delta}\leq (in)^{-\frac12}$, that is such that $k\geq {(in)}^{\frac{1}{2(\delta-1)}}-1$. With that choice we have that $(R_k+1)^{1-\delta}=\O(i^{\frac12-\delta}n^{-\frac12})$ which completes the proof.  
\end{proof}


\subsubsection{Hypothesis $\H_2$.}


In this case we put $i=n^\beta$ and $R=i^{\delta^\prime}$. We can write the order of the noise terms, using their definitions and $\H_2$, as 
\begin{align*}
\gamma_1(n,R)&\lesssim e^{-n^{\beta\delta'}\delta+(\beta\delta^\prime(s+\frac12)-\frac12)\log n}\\
\gamma_2(n,R)&\lesssim e^{-n^{\beta\delta'}\delta}+e^{-n^{\beta\delta'}\delta+(s\beta\delta^\prime-\frac12)\log{n}}
\end{align*}
The last inequality follows from the following lemma
\begin{lemma}
Under $\text{H}_2$ we have $\sqrt{\V_2(R)}=\O(e^{-\delta R}R^{s})$
\end{lemma}
\begin{proof}
Asume $R>1$. We have $\sqrt{\V_2(R)}\leq \sqrt{e^{-R\delta} \sum_{i>R}e^{-\delta i}i^{2s}}$. We can use the integral bound for the series. That is \[\sum_{i>R}e^{-\delta i}i^{2s}\leq \int^\infty_{R}e^{-\delta x}x^{2s}dx\]
Integrating by parts iteratively, it follows\[ \int^\infty_{R}e^{-\delta x}x^{2s}dx\lesssim e^{-\delta R}\sum^{2s}_{i=0}R^{2s-i}= e^{-\delta R}(R^{2s+1}-R)/(R-1)\]
From which the lemma follows.
\end{proof}
We have that Lemma \ref{lem:tau_H1} is also valid in this case, because $\V_1(\cdot)$ depends only on the eigenfunctions and the assumptions are the same that in the case $\H_1$. 

\begin{proof}[Proof Thm. \ref{thm:theo2} under $\H_2$]

Assume $s\geq 1$. If $\beta\leq \frac1{2s+1}$ holds and
given that $\lambda_i>\gamma_2$, we use Prop. \ref{prop:Rsmaller} with $\delta'=1$ to obtain \[\Err_i\lesssim_\alpha |\lambda_i|\sqrt{\frac{\V_1(R)}{n}}+\gamma_1(n,R)=\O(e^{-n^\beta\delta+(\beta(s+\frac12)-\frac12)\log n})\]
In the case $\beta>\frac1{2s}$ we use Prop. \ref{prop:Rlarger} with $\delta^\prime=1$ getting  \[\Err_i\lesssim_\alpha e^{-n^\beta\delta+(\beta s-\frac12)\log n}\]
In the case $\beta\in[\frac1{2s+1},\frac1{2s}]$ we use Prop.\ref{prop:Rlarger} with $\delta'=1$, which gives $\Err_i\lesssim_\alpha e^{-i\delta}=e^{-i\delta+\frac1{2\beta\log i}}n^{-\frac12}$. In this case, this implies that \[\Err_i\lesssim_\alpha e^{-\delta i+(s+\frac12)\log i}n^{-\frac12}\] 
Noticing that $\beta=\log{i}/\log{n}$, the result follows. 
In the case of $s=0$, we have $|\lambda_i|=\O(e^{-\delta i})$ and $\gamma_2(n,R)=\O(e^{-\delta i})+\O(e^{-\delta i}n^{-\frac12})$, which implies that $|\lambda_i|\asymp\gamma_2(n,R)$. In addition, given that $\V_1(R)/n<1$ for all $R<n$, we have that $\tau<1$.  Using Prop. \ref{prop:Rsmaller} for all the indices, with $\delta^\prime=1$ and we get 
 \[\Err_i\lesssim_\alpha e^{-i\delta+\frac12\log i-\frac12\log n}\]
\end{proof}

\subsubsection{Hypothesis $\H_3$.}



We assume that $\lambda_i=e^{-i\delta}$ and $\|\phi_i\|_{\infty}=e^{is}$. Let $i=\beta n$ and $R=\delta^\prime i$, for $\beta\leq1$. We start by giving the order of $\gamma_2$. Using the definition and the fact that $\sum_{k>R}e^{-kp}=\O(e^{-Rp})$ we have 
\begin{align*}
\gamma_1(n,R) &\lesssim  e^{R(s-\delta)} n^{-1/2}=\O(e^{n\beta \delta^\prime (s-\delta)-\frac12\log n})\\
\gamma_2(n,R)&\lesssim e^{-R\delta}+e^{R(s-\delta)}n^{-\frac12}=\O(e^{-n\beta\delta^\prime\delta})+\O(e^{n\beta\delta^\prime(s-\delta)-\frac12\log{n}})
\end{align*}
Comparing the terms in the order for $\gamma_2$ it is direct \[\gamma_2(n,R)=\begin{cases}\O(e^{-n\beta\delta^\prime\delta})\quad\text{ if }\beta\delta^\prime \leq \frac{1}{s}\frac{\log{n}}n\\ \O(e^{n\beta\delta^\prime(s-\delta)-\frac12\log{n}})\quad \text{ otherwise }  \end{cases}\]



\begin{lemma}\label{lem:t23}
If  $\beta\leq \frac{1}{s}\frac{\log n}{n}$ we have, for $\delta'=1$, that $\tau<1$ and $|\lambda_i|\asymp\gamma_2(n,R)$
\end{lemma}
\begin{proof}
If $\beta\leq \frac{1}{s}\frac{\log n}{n}$ and $\delta'=1$ we have that $e^{sn\beta\delta'}\leq n$, which implies that $\tau<1$. In addition, we see that $|\lambda_i|=O(\gamma_2(n,R))$ and $\gamma_2(n,R)=\O(|\lambda_i|)$. 
\end{proof}
\begin{proof}[Proof Thm. \ref{thm:theo2} under $\H_3$]
Consider $s\geq 1$.
If $\beta\leq \frac{1}{s}\frac{\log n}{n}$, we use Prop.\ref{prop:Rsmaller}, with $\delta'=1$, and get \[Err_i\lesssim_\alpha |\lambda_i|\sqrt{\frac{\V_1(R)}{n}}+\gamma_1(n,R)=\O(e^{-n\beta(\delta-s)}n^{-1/2})=\O(e^{-i(\delta-s)}n^{-1/2})\]

If  $\beta\leq \frac{1}{s}\frac{\log n}{n}$ we use Prop.\ref{prop:Rlarger} with $\delta=1$ to get  \[Err_i\lesssim_\alpha e^{-n\beta(\delta-2s)}n^{-1/2}=e^{-i(\delta-s)}n^{-1/2}\]

This proves the result for $s\geq 1$. The case $s=0$ is coincident with the case $\H_2$ with $s=0$.
\end{proof}



\section{Gegenbauer polynomials and spherical harmonic dimension}\label{orthogonal polynomials}

The Gegenbauer (ultraspherical) polynomials $G^\gamma_l$ are themselves multiples of the Jacobi polynomials $P^{(\gamma-\frac{1}{2},\gamma-\frac{1}{2})}_l$, satisfying  
\begin{equation}\label{eq:jac-gegembauer}
G^\gamma_l(t)=\frac{(2\gamma)^{(l)}}{(\gamma+\frac{1}{2})^{(l)}}P^{(\gamma-\frac{1}{2},\gamma-\frac{1}{2})}_l(t)
\end{equation}
where $(\cdot)^{(l)}$ is the rising Pochhammer symbol. The Jacobi polynomials are a well-studied family of orthogonal polynomials (see \cite[chap.~4]{Sze}). A convenient way to define them is as the 
 solutions of the following differential equation
\[L_{\gamma-\frac{1}{2}}P^{(\gamma-\frac{1}{2},\gamma-\frac{1}{2})}_l(t)=\beta_lP^{(\gamma-\frac{1}{2},\gamma-\frac{1}{2})}_l(t)\]
where \[L_{\gamma-\frac{1}{2}}u=-(1-t^2)^{\frac{1}{2}-\gamma}\frac{d}{dt}((1-t^2)^{\gamma+\frac{1}{2}}\frac{du}{dt})\] 
and $\beta_l=l(l+2\gamma)$. 
 The Gegenbauer polynomials satisfy the following orthogonality relations with respect to the weight function $\varrho_{\gamma}(t)=(1-t^2)^{\gamma-\frac{1}{2}}$
\[c_{\gamma}\int_{-1}^1G^{\gamma}_k(t)G^{\gamma}_l(t)\varrho_\gamma(t)dt=\frac{\gamma}{n+\gamma}G^\gamma_l(1)\delta_{kl} \]
where $c_\gamma=2^{-2\gamma}\frac{\Gamma(2\gamma+1)}{\Gamma(\gamma+\frac{1}{2})^2}$ and $G^\gamma_l(1)=\frac{(2\gamma)^{(l)}}{l!}$.
From \eqref{eq:expansion}, we have 
\begin{equation*}\label{expgegen}
p(t)=\sum_{l\geq0}\lambda_lc_lG^\gamma_l(t)=\sum_{l\geq0}\lambda_l\sqrt{d_l}\tilde{G}^\gamma_l(t)\end{equation*}
where $\tilde{G}^\gamma_l(t)=\frac{G^\gamma_l(t)}{\|G^\gamma_l\|_{L^2_\gamma}}$, $\{\lambda_l\}_{l\in\mathbb{N}}$ are the eigenvalues of the operator $T_K$ and $\{\lambda_l\sqrt{d_l}\}_{l\in\mathbb{N}}$ are the Fourier-Gegenbauer coefficients. We recall that in the case $\gamma=\frac{d-2}{2}$ we have, by \eqref{lambdas}
\[\lambda_l=\frac{\Gamma(\frac{d}{2})}{\sqrt{\pi}\Gamma(\frac{d-1}{2})}\frac{l!}{(2d-2)^{(l)}}\int_{-1}^1p(t)G^\gamma_l(t)\varrho_\gamma(t)dt\]
A useful tool to compute the previous integral is the Rodrigues formula (see \cite[eq.~4.3.1]{Sze})
\begin{equation}\label{rodri}G^\gamma_l(t)\varrho_\gamma(t)=b_{l,d}\frac{d^l}{dt^l}\varrho_{\gamma+l}(t)=b_{l,d}\frac{d^l}{dt^l}(1-t^2)^{\frac{d-3}{2}+l}\end{equation} where $b_{l,d}=\frac{(-1)^l}{2^ll!}\frac{(2d-2)^{(l)}}{\big(\frac{d-1}{2}\big)^{(l)}}$. 

\subsection{Estimation of $d_l$ coefficients}\label{dl}
The dimension of the spherical harmonic spaces is well known (see \cite[cor.~1.1.4]{Dai}),
\begin{equation}\label{formuleDls}
d_l=\binom{l+d-1}{l}-\binom{l+d-3}{l-2}.
\end{equation}
This is a polynomial in $l$. Routinary computations lead to 
\[d_l=2\frac{(l+d-3)!}{d!(l-1)!}+\binom{l+d-3}{d-3}\]
which is necessary to determine the asymptotic order with respect to $l$, because the term $2\frac{(l+d-3)!}{d!(l-1)!}$ has the leading term in $l$, which determines that $d_l=\O(l^{d-2})$.
We can also determine the order of $\kappa(R)=\sum^R_{l=0}d_l$. As we know that $d_0=1$ and $d_1=1$ we can take them out of the sum \[\kappa(R)=1+d+\sum^R_{l=2}d_l\] and 
$\sum^R_{l=2}d_l$ is of the form $\sum^R_{l=2}F(l)-F(l-2)=\sum^R_{l=2}F(l)-F(l-1)+(F(l-1)-F(l-2))$, with $F(l)=\binom{l+d-1}{l}$. Then \[\kappa(R)=1+d+F(R)+F(R-1)-F(0)-F(1)\] The leading term is contained in $F(R)$ and is a constant times $R^{d-1}$. Thus $\kappa(R)=\O(R^{d-1})$. Note that in the case of kernels that only depends on the distance, as described in Section \ref{RG} we have that the variance proxy function $\V_1$ satisfies \[\V_1(R)=\kappa(R)=\O(R^{d-1})\]

\subsection{Eigenvalues computations}\label{compu}
\subsubsection{Logistic function} \label{logisticcompu}
We will compute the eigenvalues for the logistic graphon $f(t)=\frac{e^{rt}}{1+e^{rt}}$. We recall that $\varrho_\gamma(t)=(1-t)^{\gamma-1/2}$ and $\gamma=\frac{d-2}{2}$. By \eqref{lambdas} we have
\[\lambda_l=a_{l,d}\int_{-1}^1f(t)G^\gamma_{l}(t)\varrho_\gamma(t)dt=a_{l,d}\int_{-1}^1\frac{1}{1+e^{-rt}}G^\gamma_{l}(t)\varrho_\gamma(t)dt\]
Using the Rodrigues formula for Gegenbauer polynomials eq. \eqref{rodri} we get \[\lambda_l=a_{l,d}b_{l,d}\int_{-1}^1\frac{1}{1+e^{-rt}}\frac{d^l}{dt^l}(1-t^2)^{\frac{d-3}{2}+l} dt\]
We have the following power series expansion $\frac{1}{1+e^{-rt}}=\sum^\infty_{k=0}(-1)^ke^{-krt}$, which is valid for any $t>0$. Then we get 
\begin{align*}
\lambda_l&=a_{l,d}b_{l,d}\Big(\int_{0}^1\frac{1}{1+e^{-rt}}\frac{d^l}{dt^l}\varrho_{\gamma+l}(t)+ \int_{-1}^0\frac{1}{1+e^{-rt}}\frac{d^l}{dt^l}\varrho_{\gamma+l}(t)dt\Big)\\
&=a_{l,d}b_{l,d}\Big(\int_{0}^1\frac{1}{1+e^{-rt}}\frac{d^l}{dt^l}\varrho_{\gamma+l}(t)+ \int_{0}^1\frac{1}{1+e^{rt}}\frac{d^l}{dt^l}\varrho_{\gamma+l}(-t)dt\Big)\\
&=a_{l,d}b_{l,d}\Big(\int_{0}^1\frac{1}{1+e^{-rt}}\frac{d^l}{dt^l}\varrho_{\gamma+l}(t)+ \int_{0}^1\big(1-\frac{1}{1+e^{-rt}}\big)\frac{d^l}{dt^l}\varrho_{\gamma+l}(-t)dt\Big)\\
&=a_{l,d}b_{l,d}\Big(\int^1_{0}\frac{d^l}{dt^l}\varrho_{\gamma+l}(-t)dt+\int_{0}^1\frac{1}{1+e^{-rt}}\frac{d^l}{dt^l}\big(\varrho_{\gamma+l}(t)-\varrho_{\gamma+l}(-t)\big)\Big)
\end{align*}
The parity of the Gegenbauer polynomial $G^\gamma_l(t)$ is the same that $l$, then the previous relation reduces to\[\lambda_l=a_{l,d}b_{l,d}\Big(-\int^1_{0}\frac{d^l}{dt^l}\varrho_{\gamma+l}(t)dt+\int_{0}^1\frac{2}{1+e^{-rt}}\frac{d^l}{dt^l}\varrho_{\gamma+l}(t)\Big)=a_{l,d}b_{l,d}\int^1_{0}\frac{1-e^{-rt}}{1+e^{-rt}}\frac{d^l}{dt^l}\varrho_{\gamma+l}(t)dt\]
Already from this expression we see that in the case $r\rightarrow 0$, being the rest of the parameters fixed, the eigenvalues will be \[\lambda_l=\begin{cases}
\frac{1}{2}&\text{ if }l=0\\
0&\text{ if }l>0
\end{cases}
\]
In the case $r\rightarrow \infty$ gives 
\[\lambda_l=
\begin{cases}
\frac{1}{2}&\text{ if }l=0\\
a_{l,d}b_{l,d}\int^1_0\frac{d^l}{dt^l}\varrho_{\gamma+l}(t)dt&\text{ if }l>0
\end{cases}
\]
Indeed, for $t$ fixed $g_r(t)=\frac{1-e^{rt}}{1+e^{rt}}$ satisfies $g_r(t)\rightarrow 0$ for $r=0$ and $g_r(t)=1$ for $r\rightarrow\infty$. Also, note that $g_r(t)$ is also continuous on $r$ and increasing for $r\geq 0$. 
For $0<r<\infty$, we have to compute the quantity 
\[A_{r,l}:=\int^1_0\frac{1}{1+e^{-rt}}\frac{d^l}{dt^l}w_{\gamma+l}(t)\]
 We can use the power series expansion of the logistic function
\begin{align}\label{Arl}
A_{r,l}&=\int_{0}^1\sum^\infty_{k=0}(-1)^ke^{-krt}\frac{d^l}{dt^l}\varrho_{\gamma+l}(t) dt\\
&=\sum^\infty_{k=0}(-1)^k\int_{0}^1e^{-krt}\frac{d^l}{dt^l}\varrho_{\gamma+l}(t)dt \nonumber
\end{align}
Let us first assume that $\gamma\in\mathbb{N}$, then we expand 
\begin{equation}\label{expder}
\varrho_{\gamma+l}(t)=\sum^{\gamma-\frac{1}{2}+l}_{i=0}\binom{\gamma-\frac{1}{2}+l}{i}(-1)^it^{2i}\end{equation}
 Thus 
\begin{align}\nonumber
\frac{d^l}{dt^l}\varrho_{\gamma-\frac{1}{2}+l}(t)&=\sum^{\gamma-\frac{1}{2}+l}_{i=\ceil{l/2}}\binom{\gamma-\frac{1}{2}+l}{i}(-1)^i(2i)_lt^{2i-l}\\ \label{derw}
&=\sum^{\gamma-\frac{1}{2}+\floor{l/2}}_{i=0}g_{i,l}t^{2i}
\end{align}
where $g_{i,l}:=\binom{\gamma-\frac{1}{2}+l}{i+\ceil{l/2}}(-1)^{i+\ceil{l/2}}(2i+l)_l$. Plugging into the expression \eqref{Arl} we get  
\begin{align}\nonumber
A_{r,l}&=\sum^\infty_{k=0}\sum^{\gamma+\floor{l/2}}_{i=0}(-1)^kg_{i,l}\int^1_0e^{-krt}t^{2i}dt\\ \nonumber 
&=\sum^\infty_{k=0}\sum^{\gamma+\floor{l/2}}_{i=0}\frac{(-1)^k}{(kr)^{2i+1}}g_{i,l}\int^{kr}_0e^{-t}t^{2i}dt\\ \label{Arl2} &=\sum^\infty_{k=0}\sum^{\gamma+\floor{l/2}}_{i=0}\frac{(-1)^k}{(kr)^{2i+1}}g_{i,l}\gamma(2i+1,kr)
\end{align}
On the other hand, we have\[\int_{0}^1\frac{d^l}{dt^l}\varrho_{\gamma+l}(t)=\frac{d^{l-1}}{dt^{l-1}}\varrho_{\gamma+l}(t)\bigg\rvert_{t=1}-\frac{d^{l-1}}{dt^{l-1}}\varrho_{\gamma+l}(t)\bigg\rvert_{t=0}\]
It is easy to see that $\frac{d^{l-1}}{dt^{l-1}}w_{\gamma+l}(t)\bigg\rvert_{t=1}=0$, because this derivative is a multiple of $(1-t^2)^{\gamma}$. 
By definition, we get 
\begin{align*}
\frac{d^{l-1}}{dt^{l-1}}\varrho_{\gamma+l}(t)\bigg\rvert_{t=0}&=\binom{\gamma-\frac{1}{2}+l}{\frac{l-1}{2}}(-1)^{\ceil{l/2}}(l-1)!\\
&=(-1)^{\ceil{l/2}}\frac{\Gamma(\frac{d-1}{2}+l)\Gamma(l)}{\Gamma(\frac{l+1}{2})\Gamma(\frac{d+l}{2})}
\end{align*}
On the other hand, 
\begin{align*}
a_{l,d}b_{l,d}&=\frac{\Gamma(\frac{d}{2})}{\sqrt{\pi}\Gamma(\frac{d-1}{2})}\frac{l!}{(2d-2)_l}\frac{(-1)^l}{2^ll!}\frac{(2d-2)_l}{\big(\frac{d-1}{2}\big)_l}\\
&=\frac{(-1)^l\Gamma(\frac{d}{2})}{2^l\sqrt{\pi}\Gamma(\frac{d-1}{2})(\frac{d-1}{2})_l}\\
&=\frac{(-1)^l\Gamma(\frac{d}{2})}{2^l\sqrt{\pi}\Gamma(\frac{d-1}{2}+l)}
\end{align*}
Then we have 
\begin{align*}
a_{l,d}b_{l,d}\int_0^1\frac{d^l}{dt^l}\varrho_{\gamma+l}(t)dt&=\frac{(-1)^{l+\ceil{l/2}}}{2^l\sqrt{\pi}}\frac{\Gamma(l)\Gamma(\frac{d}{2})}{\Gamma(\frac{l+1}{2})\Gamma(\frac{d+l}{2})}\\
&=\frac{(-1)^{l+\ceil{l/2}}}{2\pi }\frac{\Gamma(\frac{l}{2})\Gamma(\frac{d}{2})}{\Gamma(\frac{d+l}{2})}\\
&=\frac{(-1)^{l+\ceil{l/2}}}{2\pi}Beta\big(\frac{d}{2},\frac{l}{2}\big)
\end{align*}
The eigenvalues are for $l$ even: 
\begin{equation}\label{eigenlogistic}
\lambda_l=2\sum^\infty_{k=0}\sum^{\gamma+\floor{l/2}}_{i=0}\frac{(-1)^ka_{l,d}b_{l,d}}{(kr)^{2i+1}}g_{i,l}\gamma(2i+1,kr)+\frac{2(-1)^{\frac{3l+1}{2}}}{\pi (l+1)}Beta\big(\frac{d}{2},\frac{l}{2}+1\big) \end{equation}
In the case $\gamma\notin \mathbb{N}$ we can use the generalized binomial expansion\[\varrho_{\gamma+l}(t)=(1-t^2)^{\gamma-\frac{1}{2}+l}=\sum^\infty_{k=0}\binom{\gamma-\frac{1}{2}+l}{k}(-1)^kt^{2k}\] which is absolutely convergent for $-1\leq t\leq 1$. The generalized binomial coefficient is defined for $\alpha,\beta\in\mathbb{R}$ by $\binom{\alpha}{\beta}:=\frac{\Gamma(\alpha+1)}{\Gamma(\beta+1)\Gamma(\alpha-\beta+1)}$. With that we obtain 
\[\frac{d^l}{dt^l}\varrho_{\gamma+l}(t)=\sum^\infty_{i=0}\tilde{g}_{i,l}t^{2i}\] where $\tilde{g}_{i,l}=\binom{\gamma+l}{i+\ceil{l/2}}(-1)^{i+\ceil{l/2}}(2i+l)_l$. Note that definition is the same that in the particular case $\gamma\in\mathbb{N}$, but the binomial in $\tilde{g}_{i,l}$ is the generalized one. Of course, one might take the more general definition in the case of $g_{i,l}$ too. 
\subsubsection{Threshold function}
If $\gamma\in\mathbb{N}$ and $l$ is odd, the computations in Appendix \ref{logisticcompu} can be used for the eigenvalues of the threshold function $f(t)=\mathbf{1}_{t\leq 0}$. Indeed, 
\begin{align}
\lambda^\ast_{l}&=a_{l,d}b_{l,d}\int^1_{0}\frac{d^l}{dt^l}w_{\gamma+l}(t)dt\nonumber\\
&=\frac{(-1)^{l+\ceil{l/2}}}{2\pi }Beta\big(\frac{d}{2},\frac{l}{2}\big)\nonumber
\end{align}
If $\gamma\in\mathbb{N}$ and $l$ is even and $l\neq 0$ we have $\lambda^\ast_l=0$, because \[\int_0^1\frac{d^l}{dt^l}w_{\gamma+l}(t)=\frac{d^{l-1}}{dt^{l-1}}w_{\gamma+l}(t)\bigg\rvert_{t=1}-\frac{d^{l-1}}{dt^{l-1}}w_{\gamma+l}(t)\bigg\rvert_{t=0}=0\]
Indeed, we have $\frac{d^{l-1}}{dt^{l-1}}w_{\gamma+l}(t)\bigg\rvert_{t=1}=0$, because that derivative is a multiple of $(1-t^2)^{\gamma}$. On the other hand $\frac{d^{l-1}}{dt^{l-1}}w_{\gamma+l}(t)\bigg\rvert_{t=0}=0$, because when we derivate an odd number of times the function $w_{\gamma+l}(t)$ there will be no constant term (recall that $\frac{d^{l-1}}{dt^{l-1}}w_{\gamma+l}(t)$ is a polynomial in $t^2$). If $l=0$ we have 
\[\lambda_0=a_{l,d}\int_0^1w_{\gamma}(t)dt=\frac{1}{2}\]

\section{Additional lemmas}


\begin{lemma}\label{lem:convergenceas}
If $W$ is a kernel satisfying $\text{H}$, that is $\|\sum_{i\geq 1}|\lambda_i|\phi^2_i\|_{\infty}<\infty$, then \[\sum^n_{i=1}\lambda_i\phi_i(x)\phi_i(y)\to W(x,y)\] in the $\mu\times\mu$ almost sure sense. 
\end{lemma}
\begin{proof}
Define $I_+=\{i\in\N: \lambda_i>0\}$, $I_-=\{i\in\N: \lambda_i<0\}$ and $[n]=\{1,2,\cdots,n\}$. Given $\text{H}$, we know that $\sum_i|\lambda_i|\phi^2_i(x)<\infty$ for all $x\in\bar{\Omega}$, where $\bar{\Omega}\subset\Omega$ is a set of measure $1$. This implies that \[\sum_{i\in I_+}\lambda_i\phi^2(x)<\infty\] and \[|\sum_{i\in I_-}\lambda_i\phi^2(x)|<\infty\] for $x\in\bar{\Omega}$. Then, for any $x\in\bar{\Omega}$, we have  $\sum_{i\in I_+\cap [n]}\lambda_i\phi^2(x)<\infty$ is a non decreasing sequence in $n$ and bounded, hence convergent. Let $\epsilon>0$ be arbitrary, and $m,m'\in\N$ such that $m'>m$ and $I(m,m')=I_+\cap[m']\setminus I_+\cap[m]$, then \begin{align*}
|\sum_{i\in I(m,m')}\lambda_i\phi_i(x)\phi_i(y)|&\leq \sqrt{\sum_{i\in I(m,m')}\lambda_i\phi^2_i(x)}\sqrt{\sum_{i\in I(m,m')}\lambda_i\phi^2_i(y)} 
\end{align*} 
Given that for any $x\in\bar{\Omega}$ we have $\sum_{i\in I_+\cap [n]}\lambda_i\phi^2(x)$ is a Cauchy sequence, there exists $n_0\in\N$ such that $\sum_{i\in I(m,m')}\lambda_i\phi^2_i(x)<\epsilon$ and $\sum_{i\in I(m,m')}\lambda_i\phi^2_i(y)<\epsilon$ for $m,m'\geq n_0$, which implies that $ |\sum_{i\in I(m,m')}\lambda_i\phi_i(x)\phi_i(y)|<\epsilon$ for $m,m'\geq n_0$, thus sequence $\sum_{i\in I(m,m')}\lambda_i\phi_i(x)\phi_i(y)$ is Cauchy. Recall the $L^2$ decomposition \begin{equation}\label{eq:asconv}W(x,y)=\sum_{i\in I_+}\lambda_i\phi_i(x)\phi_i(y)+\sum_{i\in I_-}\phi_i(x)\phi_i(y)\end{equation}
So far we have prove that $\sum_{i\in I_+}\lambda_i\phi_i(x)\phi_i(y)$ converge a.s. The proof that $\sum_{i\in I_-}\phi_i(x)\phi_i(y)$ converges a.s. is analogous, which proves that the right hand side of \eqref{eq:asconv} converges almost surely. Since the almost sure limit and the $L^2$ are coincident in a set of full measure, the result follows. 
\end{proof}


\begin{lemma}\label{lem:equiv}
Let $W$ be a kernel such that $\V_1(i)=\O(i)$ for all $i$ and $\V_2(R)=\O(\sum_{i>R}|\lambda_i|)$, then the results of Theorem \ref{thm:theo2}, for $\text{H}_1$ and $\text{H}_2$, are valid with $s=0$. 
\end{lemma}
\begin{proof}
Upon inspection of the proof of Theorem \ref{thm:theo2}, the regularity hypothesis are only used to obtain estimates for $\V_1$ and $\V_2$. If the kernel satisfy the hypothesis here enunciated, then the same conclusion of Theorem \ref{thm:theo2} follows.  
\end{proof}



\end{document}